\documentclass{article} % For LaTeX2e
\pdfoutput=1
\usepackage{amsmath,amssymb,amsthm}
\usepackage{fullpage}
\usepackage{url}
\usepackage[numbers]{natbib}
\usepackage{color}
\usepackage{graphicx}

\usepackage{microtype}

\usepackage{algorithm,algorithmic}
\usepackage[usenames,dvipsnames]{xcolor}
\usepackage[colorlinks=true,linkcolor=blue,citecolor=violet,urlcolor=black]{hyperref}
\usepackage{array}

\usepackage[suppress]{color-edits}
\addauthor{df}{Red}

\usepackage{amsthm}

\usepackage{mathtools}

\usepackage{amsmath}

\usepackage{amsfonts}

\usepackage{amssymb}

\usepackage{MnSymbol}

\DeclarePairedDelimiter{\brk}{[}{]}
\DeclarePairedDelimiter{\crl}{\{}{\}}
\DeclarePairedDelimiter{\prn}{(}{)}
\DeclarePairedDelimiter{\nrm}{\|}{\|}
\DeclarePairedDelimiter{\tri}{\langle}{\rangle}
\DeclarePairedDelimiter{\dtri}{\llangle}{\rrangle}

\DeclarePairedDelimiter{\ceil}{\lceil}{\rceil}

\newcommand{\ls}{\ell}

\DeclareMathOperator*{\argmin}{arg\,min} % * seems to change location of subscript

\newcommand{\N}{\mathbb{N}}

\newcommand{\C}{\mathbb{C}}

\newcommand{\D}{\mathcal{D}}

\newcommand{\defeq}{\overset{def}{=}}

\include{multiminimax}

\makeatletter
\newcommand*\rel@kern[1]{\kern#1\dimexpr\macc@kerna}
\newcommand*\widebar[1]{%
  \begingroup
  \def\mathaccent##1##2{%
    \rel@kern{0.8}%
    \overline{\rel@kern{-0.8}\macc@nucleus\rel@kern{0.2}}%
    \rel@kern{-0.2}%
  }%
  \macc@depth\@ne
  \let\math@bgroup\@empty \let\math@egroup\macc@set@skewchar
  \mathsurround\z@ \frozen@everymath{\mathgroup\macc@group\relax}%
  \macc@set@skewchar\relax
  \let\mathaccentV\macc@nested@a
  \macc@nested@a\relax111{#1}%
  \endgroup
}
\makeatother

\let\inf\undef
\DeclareMathOperator*{\inf}{\vphantom{p}inf}
\let\sup\undef
\DeclareMathOperator*{\sup}{\vphantom{p}sup}

\newcommand{\inner}[1]{\left\langle #1 \right\rangle}

\newcommand{\E}{\ensuremath{\mathbb E}} 
\newcommand{\F}{\ensuremath{\mathcal F}} 
 
\newcommand{\cH}{\ensuremath{\mathcal H}} 

\newcommand{\cN}{\ensuremath{\mathcal N}}

\newcommand{\cZ}{\ensuremath{\mathcal Z}}

\newcommand{\mbb}[1]{\mathbb{#1}}
\newcommand{\mbf}[1]{\mathbf{#1}}
\newcommand{\mc}[1]{\mathcal{#1}}
\newcommand{\mrm}[1]{\mathrm{#1}}

\newcommand{\reals}{\ensuremath{\mathbb R}}

\def\deq{\triangleq}

\newcommand{\Rad}{\mc{Rad}}

\newcommand{\loss}{\ensuremath{\boldsymbol\ell}}

\newcommand{\norm}[1]{\left\|#1\right\|}

\newcommand{\w}{\ensuremath{\mathbf w}}  
\newcommand{\x}{\ensuremath{\mathbf x}} 
\newcommand{\y}{\ensuremath{\mathbf y}} 
\newcommand{\z}{\ensuremath{\mathbf z}} 

\newcommand{\Es}[2]{\mbb{E}_{#1}\left[ #2 \right]}

\newcommand{\A}{\mathbf{A}}

\renewcommand{\v}{\ensuremath{\mathbf{v}}}

\newcommand{\eps}[1]{\epsilon_{\mathrm{#1}}}

\newcommand{\G}{\mathcal{G}}

\newcommand{\X}{\mathcal{X}}
\newcommand{\Y}{\mathcal{Y}}

\renewcommand{\v}{\mathbf{v}}

\renewcommand{\Rad}{\mc{R}}

\newcommand{\B}{\mc{B}}

\renewcommand{\X}{\mc{X}}

\renewcommand{\z}{\mbf{z}}
\newcommand{\En}{\mbb{E}}

\renewcommand{\F}{\mc{F}}

\newcommand{\V}{V}

\usepackage{amsmath,amsfonts,amssymb,amsthm}
\usepackage{graphicx}
\usepackage{color}
\usepackage{tikz}
\usetikzlibrary{arrows,shapes}
\usetikzlibrary{patterns}

\newtheorem{theorem}{Theorem}
\newtheorem{lemma}[theorem]{Lemma}

\newtheorem{definition}{Definition}

\newtheorem{corollary}[theorem]{Corollary}

\newtheorem{proposition}[theorem]{Proposition}

\newtheorem{example}{Example}[section]

\newcommand{\R}{\ensuremath{\mathbb{R}}}

\newcommand{\ip}[2]{{\left\langle{#1},{#2}\right\rangle}}
\newcommand{\abs}[1]{\left\lvert{#1}\right\rvert}

\let\abs\undefined
\DeclarePairedDelimiter{\abs}{\lvert}{\rvert}

\newcommand{\xr}[1][n]{x_{1:#1}}
\newcommand{\yr}[1][n]{y_{1:#1}}

\renewcommand{\A}{\mathcal{A}}
\renewcommand{\V}{\mathcal{V}}
\renewcommand{\C}{\mathcal{C}}
\renewcommand{\defeq}{\triangleq}
\renewcommand{\eps}{\epsilon}

\newcommand{\Relax}{\mathbf{Rel}}
\newcommand{\Relaxn}{\Relax_{n}}
\newcommand{\Rd}{\mathbf{R}}
\newcommand{\Ada}{\mathbf{Ada}_{n}}
\newcommand{\Bay}{\Relaxn}

\newcommand{\KL}{\text{KL}}

\title{Adaptive Online Learning}

\author{
Dylan J. Foster\\
Cornell University\\
\and
Alexander Rakhlin \\
University of \dfedit{Pennsylvania} \\
\and
Karthik Sridharan \\
Cornell University \\
}
\date{}

\newcommand{\multiminimax}[1]{\ensuremath{\left\llangle #1\right\rrangle}}
\renewcommand{\E}[1]{\mathbb{E}\left[#1 \right]}
\begin{document}

\maketitle

\begin{abstract}
We propose a general framework for studying adaptive regret bounds in the online learning framework, including model selection bounds and data-dependent bounds. Given a data- or model-dependent bound we ask, ``Does there exist some algorithm achieving this bound?'' We show that modifications to recently introduced sequential complexity measures can be used to answer this question by providing sufficient conditions under which adaptive rates can be achieved. In particular each adaptive rate induces a set of so-called offset complexity measures, and obtaining small upper bounds on these quantities is sufficient to demonstrate achievability. A cornerstone of our analysis technique is the use of one-sided tail inequalities to bound suprema of offset random processes.\\
\indent Our framework recovers and improves a wide variety of adaptive bounds including quantile bounds, second-order data-dependent bounds, and small loss bounds. In addition we derive a new type of adaptive bound for online linear optimization based on the spectral norm, as well as a new online PAC-Bayes theorem that holds for countably infinite sets. 
\end{abstract}

% !TEX root = paper.tex

\section{Introduction}

Some of the recent progress on the theoretical foundations of \emph{online learning} has been motivated by the parallel developments in the realm of \emph{statistical learning}. In particular, this motivation has led to martingale extensions of empirical process theory, which were shown to be the ``right'' notions for online learnability. Two topics, however, have remained elusive thus far: obtaining data-dependent bounds and establishing model selection (or, oracle-type) inequalities for online learning problems. In this paper we develop new techniques for addressing both these questions.

Oracle inequalities and model selection have been topics of intense research in statistics in the last two decades \cite{birge1998minimum,lugosi1999adaptive,bartlett2002model}. Given a sequence of  models $\mathcal{M}_1, \mathcal{M}_2, \ldots$ whose union is $\mathcal{M}$, one aims to derive a procedure that selects, given an i.i.d. sample of size $n$, an estimator $\hat{f}$ from a model $\mathcal{M}_{\hat{m}}$ that trades off bias and variance. Roughly speaking the desired oracle bound takes the form
$$\text{err}(\hat{f})\leq  \inf_{m}\left\{ \inf_{f\in\mathcal{M}_m} \text{err}(f) + \text{pen}_n(m) \right\},$$
where $\text{pen}_n(m)$ is a penalty for the model $m$. Such oracle inequalities are attractive because they can be shown to hold even if the overall model $\mathcal{M}$ is too large. A central idea in the proofs of such statements (and an idea that will appear throughout the present paper) is that $\text{pen}_n(m)$ should be ``slightly larger'' than the fluctuations of the empirical process for the model $m$. It is therefore not surprising that concentration inequalities---and particularly Talagrand's celebrated inequality for the supremum of the empirical process---have played an important role in attaining oracle bounds. In order to select a good model in a data-driven manner, one establishes non-asymptotic data-dependent bounds on the fluctuations of an empirical process indexed by elements in each model \cite{massart2007concentration}.

Lifting the ideas of oracle inequalities and data-dependent bounds from statistical to online learning is not an obvious task. For one, there is no concentration inequality available, even for the simple case of  sequential Rademacher complexity. (For the reader already familiar with this complexity: a change of the value of one Rademacher variable results in a change of the remaining path, and hence an attempt to use a version of a bounded difference inequality grossly fails). Luckily, as we show in this paper, the concentration machinery is not needed and one only requires a one-sided tail inequality. This realization is motivated by the recent work of \cite{Mendelson14,offset2015,onr2014}. At a high level, our approach will be to develop one-sided inequalities for the suprema of certain offset processes \cite{onr2014}, where the offset is chosen to be ``slightly larger'' than the complexity of the corresponding model. We then show that these offset processes determine which data-dependent adaptive rates are achievable for online learning problems, drawing strong connections to the ideas of statistical learning described earlier.
\subsection{Framework}
\label{sec:framework}

Let $\X$ be the set of observations, $\D$ the space of decisions, and $\Y$ the set of outcomes. Let $\Delta(S)$ denote the set of distributions on a set $S$. Let $\ell:\D\times\Y\to\reals$ be a loss function. The online learning framework is defined by the following process: For $t=1,\ldots{},n$, Nature provides input instance $x_{t}\in{}\X$; Learner selects prediction distribution $q_{t}\in{}\Delta(\D)$; Nature provides label $y_{t}\in{}\Y$, while the learner draws prediction $\hat{y}_t\sim{}q_t$ and suffers loss $\ls(\hat{y}_t, y_t)$.

Two important settings are \emph{supervised learning} ($\Y\subseteq\reals$, $\D\subseteq\reals$) and \emph{online linear optimization} ($\X=\crl{0}$ is a singleton set, $\Y$ and $\D$ are balls in dual Banach spaces and $\loss(\hat{y},y) = \inner{\hat{y},y}$). For a class $\F\subseteq\D^\X$, we define the learner's cumulative \emph{regret} to $\F$ as 
\[
\sum_{t=1}^{n}\ls(\hat{y}_t, y_t) - \inf_{f\in{}\F}\sum_{t=1}^{n}\ls(f(x_t), y_t).
\]
A \emph{uniform} regret bound $\B_n$ is achievable if there exists a randomized algorithm selecting $\hat{y}_t$ such that
\begin{align}
	\label{eq:uniform_bound}
\En\brk*{\sum_{t=1}^{n}\ls(\hat{y}_t, y_t) - \inf_{f\in{}\F}\sum_{t=1}^{n}\ls(f(x_t), y_t)} \leq \B_n \quad\forall{}x_{1:n},y_{1:n},
\end{align}
where $a_{1:n}$ stands for $\{a_1,\ldots,a_n\}$. Achievable rates $\B_n$ depend on complexity of the function class $\F$. For example, sequential Rademacher complexity of $\F$ is one of the tightest achievable uniform rates for a variety of loss functions \cite{RST10,onr2014}.

An \emph{adaptive regret bound} has the form $\B_{n}(f; {}\xr[n], \yr[n])$ and is said to be achievable if there exists a randomized algorithm for selecting $\hat{y}_t$ such that
\begin{align}
	\label{eq:adaptive_bound}
\En\brk*{\sum_{t=1}^{n}\ls(\hat{y}_t, y_t) -\sum_{t=1}^{n}\ls(f(x_t), y_t)} \leq \B_n(f; \xr[n], \yr[n])\quad\forall{}x_{1:n},y_{1:n},\;\forall{}f\in{}\F.
\end{align}
We distinguish three types of adaptive bounds, according to whether $\B_{n}(f; {}\xr[n], \yr[n])$ depends only on $f$, only on $(\xr[n], \yr[n])$, or on both quantities. Whenever $\B_n$ depends on $f$, an adaptive regret can be viewed as an oracle inequality which penalizes each $f$ according to a measure of its complexity (e.g. the complexity of the smallest model to which it belongs). As in statistical learning, an oracle inequality \eqref{eq:adaptive_bound} may be proved for certain functions $\B_n(f; \xr[n], \yr[n])$ even if a uniform bound \eqref{eq:uniform_bound} cannot hold for any nontrivial $\B_n$.
\subsection{Related Work}

The case when $\B_n(f;\xr[n],\yr[n])=\B_n(\xr[n],\yr[n])$ does not depend on $f$ has received most of the attention in the literature. The focus is on bounds that can be tighter for ``nice sequences,'' yet maintain near-optimal worst-case guarantees. An incomplete list of prior work includes  \cite{hazan2010extracting,Chiangetal12,RakSri13pred,duchi2011adaptive}, couched in the setting of online linear/convex optimization, and \cite{cesa2007improved} in the experts setting. 

A bound of type $\B_{n}(f)$ was studied in \cite{chaudhuri2009parameter}, which presented an algorithm that competes with all experts simultaneously, but with varied regret with respect to each of them depending on the quantile of the expert. Another bound of this type was given by \cite{McMahan2014}, who consider online linear optimization with an unbounded set and provide oracle inequalities with an appropriately chosen function $\B_n(f)$. 

Finally, the third category of adaptive bounds are those that depend on both the hypothesis $f \in \F$ and the data. The bounds that depend on the loss of the best function (so-called ``small-loss'' bounds,\linebreak \cite[Sec. 2.4]{PLG}, \cite{srebro2010smoothness,cesa2007improved}) fall in this category trivially, since one may overbound the loss of the best function by the performance of $f$. We draw attention to the recent result of \cite{LuoSch15} who show an adaptive bound in terms of both the loss of comparator and the KL divergence between the comparator and some pre-fixed prior distribution over experts. An MDL-style bound in terms of the variance of the loss of the comparator (under the distribution induced by the algorithm) was recently given in \cite{Koolen15}.

Our study was also partly inspired by Cover \cite{Cover67} who characterized necessary and sufficient conditions for achievable bounds in prediction of binary sequences. The methods in \cite{Cover67}, however, rely on the structure of the binary prediction problem and do not readily generalize to other settings.

The framework we propose recovers the vast majority of known adaptive rates in literature, including variance bounds, quantile bounds, localization-based bounds, and fast rates for small losses.
It should be noted that while existing literature on adaptive online learning has focused on simple hypothesis classes such as finite experts and finite-dimensional $p$-norm balls, our results extend to general hypothesis classes, including large nonparametric ones discussed in \cite{onr2014}.

% !TEX root =  paper.tex

\section{Adaptive Rates and Achievability: General Setup}
The first step in building a general theory for adaptive online learning is to identify what adaptive regret bounds are possible to achieve. Recall that an adaptive regret bound of $\B_n: \F \times \X^n \times \Y^n \to \reals$ is said to be achievable if there exists an online learning algorithm such that, \eqref{eq:adaptive_bound} holds.

In the rest of this work, we use the notation $\multiminimax{\ldots}_{t=1}^n$ to denote the interleaved application of the operators inside the brackets, repeated over $t = 1,\ldots,n$ rounds (see \cite{StatNotes2012}). Achievability of an adaptive rate can be formalized by the following minimax quantity. 
\begin{definition}
Given an adaptive rate $\B_{n}$ we define the offset minimax value:
{\small
\[
\A_{n}(\F,\B_n) \defeq \dtri*{\sup_{x_t\in{}\X}\inf_{q_t\in{}\Delta(\D)}\sup_{y_t\in{}\Y}\underset{\hat{y}_t\sim{}q_t}{\En}}_{t=1}^{n}
\brk*{
\sum_{t=1}^{n}\ls(\hat{y}_t, y_t) -\inf_{f\in{}\F}\left\{\sum_{t=1}^{n}\ls(f(x_t), y_t) + \B_n(f ; \xr[n], \yr[n])\right\}
}.
\]
}
\end{definition}
$\A_{n}(\F,\B_n)$ quantifies how $\sum_{t=1}^{n}\ls(\hat{y}_t, y_t) -\inf_{f\in{}\F}\left\{\sum_{t=1}^{n}\ls(f(x_t), y_t) + \B_n(f; \xr[n], \yr[n])\right\}$ behaves when the optimal learning algorithm that minimizes this  difference is used against Nature trying to maximize it. Directly from this definition,
$$
\textrm{\bf An adaptive rate }\B_n\textrm{ \bf is achievable if and only if }\A_{n}(\F, \B_n)\leq 0.
$$
If $\B_n$ is a uniform rate, i.e., $\B_n(f; x_{1:n},y_{1:n}) = \B_n$, achievability reduces to the minimax analysis explored in \cite{RST10}. The uniform rate $\B_n$ is achievable if and only if $\B_n \geq\V_{n}(\F)$, where $\V_{n}(\F)$ is the minimax value of the online learning game. 

We now focus on understanding the minimax value $\A_{n}(\F,\B_n)$ for general adaptive rates. We first show that the minimax value is bounded by an offset version of the sequential Rademacher complexity studied in \cite{RST10}. The symmetrization Lemma~\ref{lem:sym} below provides us with the first step towards a probabilistic analysis of achievable rates. Before stating the lemma, we need to define the notion of a tree and the notion of sequential Rademacher complexity. 

Given a set $\cZ$, a $\cZ$-valued tree $\z$ of depth $n$ is a sequence $(\z_t)_{t=1}^n$ of functions $\z_t:\{\pm{}1\}^{t-1}\to\cZ$. One may view $\z$ as a complete binary tree decorated by elements of $\cZ$. Let $\epsilon=(\epsilon_t)_{t=1}^n$ be a sequence of independent Rademacher random variables. Then $(\z_t(\epsilon))$ may be viewed as a predictable process with respect to the filtration $\mathcal{Sigma}_t=\sigma(\epsilon_1,\ldots,\epsilon_{t})$. For a tree $\z$, the sequential Rademacher complexity of a function class $\G\subseteq \reals^\cZ$ on $\z$ is defined as 
$$ \Rad_n(\G, \z) \defeq \En_\epsilon \sup_{g\in\G} \sum_{t=1}^n \epsilon_t g(\z_t(\epsilon)) ~~~~~\textrm{and}~~~~~\Rad_n(\G) \deq \sup_{\z} \Rad_n(\G,\z)~.$$

%\begin{lemma}\label{lem:sym}
%For any lower semi-continuous loss $\ell$, and any adaptive rate $\B_n$ that only depends on outcomes (i.e. $\B_n(f;x_{1:n},y_{1:n}) = \B_n(y_{1:n})$), we have that
%\begin{align}\label{eq:data}
%\A_n \leq \sup_{\x, \y}\Es{\epsilon}{\sup_{f\in{}\F}\left\{
%2\sum_{t=1}^{n}\epsilon_{t}\ls(f(\x_{t}(\epsilon)), \y_{t}(\epsilon)) \right\}- \B_{n}(\y_{1:n}(\epsilon))}.
%\end{align}
%Further, for any general adaptive rate $\B_n$,
%\begin{align}\label{eq:general}
%\A_n \leq \sup_{\x, \y, \y'}\Es{\epsilon}{\sup_{f\in{}\F}\left\{
%2 \sum_{t=1}^{n}\epsilon_{t} \ell(f(\x_{t}(\epsilon)),\y_t(\epsilon))  - \B_{n}(f;\x_{1:n}(\epsilon),\y'_{2:n+1}(\epsilon))\right\}}.
%\end{align}
%Finally, if one considers the supervised learning problem where $\F : \X \to \reals$, $\Y \subset \reals$ and $\ell: \reals \times \reals \to \reals$ is a loss that is convex and $L$-Lipschitz in its first argument, then for any adaptive rate $\B_n$, 
%\begin{align}\label{eq:convlip}
%\A_n \leq \sup_{\x, \y}\Es{\epsilon}{\sup_{f\in{}\F}\left\{
%2 L \sum_{t=1}^{n}\epsilon_{t} f(\x_{t}(\epsilon)) - \B_{n}(f;\x_{1:n}(\epsilon),\y_{1:n}(\epsilon))\right\}}.
%\end{align}
%\end{lemma}
\begin{lemma}\label{lem:sym}
For any lower semi-continuous loss $\ell$, and any adaptive rate $\B$,
\begin{align}\label{eq:general}
\A_{n}(\F,\B) \leq \sup_{\x, \y, \y'}\En_{\eps}\brk*{\sup_{f\in{}\F}\left\{
2 \sum_{t=1}^{n}\epsilon_{t} \ell(f(\x_{t}(\epsilon)),\y_t(\epsilon))  - \B(f;\x_{1:n}(\epsilon),\y'_{2:n+1}(\epsilon))\right\}}.
\end{align}
If one considers the supervised learning problem where $\F : \X \to \reals$, $\Y \subset \reals$ and $\ell: \reals \times \reals \to \reals$ is a loss that is convex and $L$-Lipschitz in its first argument, then for any adaptive rate $\B$, 
\begin{align}\label{eq:convlip}
\A_{n}(\F,\B) \leq \sup_{\x, \y}\En_{\eps}\brk*{\sup_{f\in{}\F}\left\{
2 L \sum_{t=1}^{n}\epsilon_{t} f(\x_{t}(\epsilon)) - \B(f;\x_{1:n}(\epsilon),\y_{1:n}(\epsilon))\right\}}.
\end{align}
\end{lemma}
The above lemma tells us that to check whether an adaptive rate is achievable, it is sufficient to check that the corresponding adaptive sequential complexity measures are non-positive. We remark that if the above complexities are bounded by some positive quantity of a smaller order, one can form a new achievable rate $\B'_n$ by adding the positive quantity to $\B_n$.

% !TEX root = paper.tex

\section{Probabilistic Tools}\label{sec:probtool}

As mentioned in the introduction, our technique rests on certain one-sided probabilistic inequalities. We now state the first building block: a rather straightforward maximal inequality. 

\begin{proposition}\label{prop:main}
	Let $I=\{1,\ldots,N\}$, $N\leq \infty$, be a set of indices and let $(X_i)_{i\in I}$ be a sequence of random variables satisfying the following tail condition: for any $\tau > 0$, 
	\begin{align}
		\label{eq:tails}
	 P(X_{i} - B_i > \tau) \le C_1 \exp\left(- \tau^2/(2 \sigma_i^2) \right) +  C_2 \exp\left(- \tau s_i \right)
	\end{align}
	for some positive sequence $(B_i)$, nonnegative sequence $(\sigma_i)$ and nonnegative sequence $(s_i)$ of numbers, and for constants $C_1,C_2\geq 0$. Then for any $\bar{\sigma}\leq \sigma_1$, $\bar{s}\geq s_1$, and 
	$$\theta_i =  \max\left\{ \frac{\sigma_i}{B_i}\sqrt{2\log(\sigma_i/\bar{\sigma}) + 4 \log(i)}  ,  (B_i s_i)^{-1}\log\left(i^2 (\bar{s}/s_i)\right)\right\} +1,$$
\begin{align}
	\label{eq:maximal_inequality}
\textrm{it holds that } ~~~~~~~~~~~~~~~~~~~~~~~~~~~~~~~~~~ \En \sup_{i \in I}\left\{ X_i - B_i \theta_i\right\}  \le 3C_1 \bar{\sigma} + 2C_2 (\bar{s})^{-1}.~~~~~~~~~~~~~~~~~~~~~~~~~~~~~~~~
\end{align}	
\end{proposition}
We remark that $B_i$ need not be the expected value of $X_i$, as we are not interested in two-sided deviations around the mean. 

One of the approaches to obtaining oracle-type inequalities is to split a large class into smaller ones according to a ``complexity radius'' and control a certain stochastic process separately on each subset (also known as the \emph{peeling} technique). In the applications below, $X_i$ will often stand for the (random) supremum of this process on subset $i$, and $B_i$ will be an upper bound on its typical size. Given deviation bounds for $X_i$ above $B_i$, the dilated size $B_i\theta_i$ then allows one to pass to maximal inequalities \eqref{eq:maximal_inequality} and thus verify achievability in Lemma~\ref{lem:sym}. The same strategy works for obtaining data-dependent bounds, where we first prove tail bounds for the given size of the data-dependent quantity, then appeal to \eqref{eq:maximal_inequality}.

A simple yet powerful example for the control of the supremum of a stochastic process is an inequality due to Pinelis \cite{Pinelis} for the norm (which is a supremum over the dual ball) of a martingale in a 2-smooth Banach space. Here we state a version of this result that can be found in \cite[Appendix A]{RakSriTew10beyond_arxiv}.

\begin{lemma}
	\label{lem:pinelis}
	Let $\cZ$ be a unit ball in a separable $(2,D)$-smooth Banach space $\cH$. For any $\cZ$-valued tree $\z$, and any $n> \tau/4D^{2}$
{\small	$$P\left ( \left\|\sum_{t=1}^n \epsilon_t \z_t(\epsilon)\right\| \geq \tau \right)\leq 2\exp\left( -\frac{\tau^2}{8D^{2}n}\right)$$}
\end{lemma}

When the class of functions is not linear, we may no longer appeal to the above lemma. Instead, we make use of a result from \cite{RakSriTew14ptrf} that extends Lemma~\ref{lem:pinelis} at a price of a poly-logarithmic factor. Before stating this lemma, we briefly define the relevant complexity measures (see \cite{RakSriTew14ptrf} for more details). First, a set $V$ of $\reals$-valued trees is called an $\alpha$-cover of $\G\subseteq \reals^\cZ$ on $\z$ with respect to $\ell_p$ if
$$\forall g\in\G, \forall \epsilon\in\{\pm1\}^n, \exists \v\in V ~~\text{ s.t. }~~ \sum_{t=1}^n (g(\z_t(\epsilon))-\v_t(\epsilon))^p \leq n\alpha^p.$$
The size of the smallest $\alpha$-cover is denoted by $\cN_p(\G, \alpha, \z)$, and $\cN_p(\G,\alpha, n) \deq \sup_\z \cN_p(\G,\alpha,\z)$. 

The set $V$ is an $\alpha$-cover of $\G$ on $\z$ with respect to $\ell_{\infty}$ if
\[
\forall{}g\in{}\G,\forall{}\eps\in{}\crl{\pm{}1},\exists{}\v\in{}V ~~\text{ s.t. }~~ \abs{g(\z_{t}(\eps)) - \v_{t}(\eps)} \leq \alpha\quad\forall{}t\in{}\brk{n}.
\]
We let $\cN_\infty (\G, \alpha, \z)$ be the smallest such cover and set $\cN_{\infty}(\G, \alpha, n) = \sup_{\z}\cN_{\infty}(\G, \alpha, \z)$.
\begin{lemma}[\cite{RakSriTew14ptrf}]
	\label{lem:dudley_probability}
	Let $\G\subseteq[-1,1]^\cZ$. Suppose $\Rad_n(\G)/n\to 0$ with $n\to\infty$ and that the following mild assumptions hold: $\Rad_n(\G)\geq 1/n$, $\cN_\infty(\G,2^{-1}, n)\geq 4$, and there exists a constant $\Gamma$ such that $\Gamma \geq \sum_{j=1}^\infty \cN_\infty(\G,2^{-j},n)^{-1}$.
	Then for any $\theta>\sqrt{12/n}$, for any $\cZ$-valued tree $\z$ of depth $n$,
	\begin{align*}
	&P \left( \sup_{g\in\G} \left|\sum_{t=1}^n \epsilon_t g(\z_t(\epsilon))\right| > 8\left(1+\theta \sqrt{8n \log^{3}(en^2) }\right)\cdot \Rad_n(\G) \right) \\
	&~~~\leq P \left( \sup_{g\in\G} \left|\sum_{t=1}^n \epsilon_t g(\z_t(\epsilon))\right| > n\inf_{\alpha > 0}\left\{ 4 \alpha + 6 \theta \int_{\alpha}^{1} \sqrt{\log \cN_\infty(\G,\delta,n)} d \delta \right\} \right) \leq 2\Gamma e^{- \frac{n \theta^2  }{4}}.
	\end{align*}
\end{lemma}
The above lemma yields a one-sided control on the size of the supremum of the sequential Rademacher process, as required for our oracle-type inequalities.

Next, we turn our attention to an offset Rademacher process, where the supremum is taken over a collection of negative-mean random variables. The behavior of this offset process was shown to govern the optimal rates of convergence for online nonparametric regression \cite{onr2014}. Such a one-sided control of the supremum will be necessary for some of the data-dependent upper bounds we develop.
\begin{lemma}\label{lem:proboffset}
	Let $\z$ be a $\cZ$-valued tree of depth $n$, and let $\G\subseteq \reals^\cZ$. For any $\gamma\geq 1/n$ and $\alpha >0$,
{\small 
\begin{align*}
P&\left(\sup_{\substack{g \in \G}} \sum_{t=1}^n \left(\epsilon_t g(\z_t(\epsilon)) - 2 \alpha  g^2(\z_t(\epsilon)) \right) -  \frac{ \log \cN_2(\G, \gamma, \z)}{\alpha}  - 12\sqrt{2} \int_{1/n}^{\gamma} \sqrt{n \log \cN_2(\G,\delta, \z)} d \delta  - 1> \tau\right) \\
&~~~~~~~ \le \Gamma \exp\left(-  \frac{\tau^2}{2\sigma^2}  \right) +  \exp\left(- \frac{\alpha \tau}{2}\right),
\end{align*}}
where $\Gamma\geq \sum_{j=1}^{\log_{2}(2n\gamma)} \cN_2(\G, 2^{-j}\gamma, \z)^{-2}$ and $\sigma= 12\int_{\frac{1}{n}}^{\gamma}  \sqrt{n \log \cN_2(\G,\delta,\z)} d\delta$.
\end{lemma}
We observe that the probability of deviation has both subgaussian and subexponential components.

Using the above result and Proposition~\ref{prop:main} leads to useful bounds on the quantities in Lemma \ref{lem:sym} for specific types of adaptive rates. Given a tree $\z$, we obtain a bound on the expected size of the sequential Rademacher process when we subtract off the data-dependent $\ell_2$-norm of the function on the tree $\z$, adjusted by logarithmic terms.
\begin{corollary}
	\label{cor:small_loss_supervised}
	Suppose $\G\subseteq [-1,1]^\cZ$, and let $\z$ be any $\cZ$-valued tree of depth $n$. Assume $\log \cN_2(\G,\delta,n)\leq \delta^{-p}$ for some $p<2$. Then
 {\small\begin{align*}
\En\sup_{g \in \G, \gamma} \Biggl\{ \sum_{t=1}^n \epsilon_t g(\z_t(\epsilon))- 4 \sqrt{2(\log n) \log \cN_2(\G,\gamma/2,\z) \left( \sum_{t=1}^n g^2(\z_t(\epsilon)) +1 \right) } & \\- 24\sqrt{2} \log n\int_{1/n}^{\gamma} \sqrt{n \log \cN_2(\G,\delta,\z)}& d \delta \Biggr\}  \le 7+ 2\log n~.
\end{align*}}
\end{corollary}
The next corollary yields slightly faster rates than Corollary \ref{cor:small_loss_supervised} when $\abs{\G} < \infty$.
\begin{corollary}
	\label{cor:small_loss_finite}
	Suppose $\G\subseteq [-1,1]^\cZ$ with $\abs{\G} = N$, and let $\z$ be any $\cZ$-valued tree of depth $n$. Then
	\begin{small}
	\[
\En\sup_{g\in{}\G}\crl*{\sum_{t=1}^{n}\eps_{t}g(\z_{t}(\eps)) - 2\log\prn*{\log{}N\sum_{t=1}^{n}g^2(\z(\eps)) + e}\sqrt{32\prn*{\log{}N\sum_{t=1}^{n}g^2(\z(\eps)) + e}}} \leq{} 1.
	\]
	\end{small}
\end{corollary}

\section{Achievable Bounds}
In this section we use Lemma \ref{lem:sym} along with the probabilistic tools from the previous section to obtain an array of achievable adaptive bounds for various online learning problems. We subdivide the section into one subsection for each category of adaptive bound described in Section~\ref{sec:framework}.

\subsection{Adapting to Data}
Here we consider adaptive rates of the form $\B_n(x_{1:n},y_{1:n})$,  uniform over $f \in \F$. We show the power of the developed tools on the following example. 

\begin{example}[\textbf{Online Linear Optimization in $\reals^d$}]
Consider the problem of online linear optimization where $\F = \{x\mapsto\tri*{w,x}:w\in\R^{d},\nrm*{w}_{2}\leq{}1\}$, $\X = \{x\in\R^{d}:\nrm*{x}_{2}\leq{}1\}$, and $\ls$ is a $1$-Lipschitz loss. The following adaptive rate is achievable:
\begin{small}\[
\B_n(x_{1:n}) = 16 \sqrt{d} \log(n) \norm{\left({\textstyle\sum_{t=1}^n}  x_t x_t^\top\right)^{1/2}}_\sigma + 16 \sqrt{d} \log(n), 
\]
\end{small}where $\norm{\cdot}_\sigma$ is the spectral norm. Let us deduce this result from  Corollary~\ref{cor:small_loss_supervised}. First, observe that
\begin{small}
\[
\norm{\left({\textstyle\sum_{t=1}^n}   x_t x_t^\top\right)^{1/2}}_\sigma 
%= \sup_{w : \norm{w}_2 \le 1} \norm{\left({\textstyle\sum_{t=1}^n}  x_t x_t^\top\right)^{1/2} w} 
= \sup_{w : \norm{w}_2 \le 1} \sqrt{w^\top {\textstyle\sum_{t=1}^n}   x_t x_t^\top w} = \sup_{f \in \F} \sqrt{{\textstyle\sum_{t=1}^n}  f^{2}(x_t)}.
\]
\end{small}The linear function class $\F$ can be covered point-wise at any scale $\delta$ with $(3/\delta)^{d}$ balls and thus $\cN(\ell \circ \F, 1/(2n),\z) \le (6n)^{d}$ for any $\Y$-valued tree $\z$. We apply Corollary \ref{cor:small_loss_supervised} with $\gamma = 1/n$ and the integral term in the corollary vanishes, yielding the claimed statement.

\end{example}

\subsection{Model Adaptation}
In this subsection we focus on achievable rates for oracle inequalities and model selection, but without dependence on data. The form of the rate is therefore $\B_n(f)$. Assume we have a class $\F = \bigcup_{R \geq 1} \F(R)$, with the property that $\F(R) \subseteq \F(R')$ for any $R \le R'$. If we are told by an oracle that regret will be measured with respect to those hypotheses $f \in \F$ with $R(f) \deq \inf\{R : f \in \F(R)\} \le R^*$, then using the minimax algorithm one can guarantee a regret bound of at most the sequential Rademacher complexity $\Rad_{n}(\F(R^*))$. On the other hand, given the optimality of the sequential Rademacher complexity for online learning problems for commonly encountered losses, we can argue that for any $f \in \F$ chosen in hindsight, one cannot expect a regret better than order $\Rad_n(\F(R(f)))$. In this section we show that simultaneously for all $f \in \F$, one can attain an adaptive upper bound of $O\left( \Rad_n(\F(R(f))) \sqrt{\log \left(\Rad_n(\F(R(f))) \right)}\log^{3/2}n\right)$. That is, we may predict as if we knew the optimal radius, at the price of a logarithmic factor. This is the price of adaptation.

\begin{corollary}
\label{cor:generic_rademacher}
For any class of predictors $\F$ with $\F(1)$ non-empty, if one considers the supervised learning problem with $1$-Lipschitz loss $\ell$, the following rate is achievable:
\begin{small}
\begin{align*}
\B_{n}(f) &= \log^{3/2} n\left(K_1 \Rad_n(\F(2R(f))) \left( 1 +  \sqrt{\log\left( \frac{\log(2R(f)) \cdot\Rad_n(\F(2R(f)))}{\Rad_n(\F(1))} \right)}\right) + K_2 \Gamma \Rad_n(\F(1))\right),
\end{align*}
\end{small}for absolute constants $K_1,K_2$, and $\Gamma$ defined in Lemma~\ref{lem:dudley_probability}.
\end{corollary}
In fact, this statement is true more generally with $\F(2R(f))$ replaced by $\ell \circ \F(2R(f))$. It is tempting to attempt to prove the above statement with the exponential weights algorithm running as an aggregation procedure over the solutions for each $R$. In general, this approach will fail for two reasons. 
First, if function values grow with $R$, the exponential weights bound will scale linearly with this value. Second, an experts bound yields only a slower $\sqrt{n}$ rate.% which spoils any faster rates one may obtain using offset Rademacher complexities. 

As a special case of the above lemma, we obtain an \emph{online PAC-Bayesian theorem}. We postpone this example to the next sub-section where we get a \emph{data-dependent} version of this result. We now provide a bound for online linear optimization in $2$-smooth Banach spaces that automatically adapts to the norm of the comparator. To prove it, we use the concentration bound from \cite{Pinelis} (Lemma \ref{lem:pinelis}) within the proof of the above corollary to remove the extra logarithmic factors.

\begin{example}[\textbf{Unconstrained Linear Optimization}]
\label{eg:linopt_bound}
Consider linear optimization with $\Y$ being the unit ball of some reflexive Banach space with norm $\norm{\cdot}_*$. Let $\F = \D$ be the dual space and the loss $\ell(\hat{y},y) = \ip{\hat{y}}{y}$ (where we are using $\ip{\cdot}{\cdot}$ to represent the linear functional in the first argument to the second argument). Define $\F(R) = \crl*{f \mid{} \nrm{f}\leq{}R}$ where $\norm{\cdot}$ is the norm dual to $\norm{\cdot}_*$. If the unit ball of $\Y$ is $(2,D)$-smooth, then the following rate is achievable for all $f$ with $\nrm{f}\geq{}1$:
\[
\B(f) = D\sqrt{n}\prn*{8\nrm{f}\prn*{
1 + \sqrt{\log(2\nrm{f}) + \log\log(2\nrm{f})
}
}
+12}.
\]
For the case of a Hilbert space, the above bound was achieved by \cite{McMahan2014}.
\end{example}

% !TEX root = paper.tex

\subsection{Adapting to Data and Model Simultaneously}
We now study achievable bounds that perform online model selection in a data-adaptive way. Of specific interest is our online optimistic PAC-Bayesian bound. This bound should be compared to \cite{LuoSch15,Koolen15}, with the reader noting that it is independent of the number of experts, is algorithm-independent, and depends quadratically on the expected loss of the expert we compare against.

%The bound simultaneously adapts to the loss of the mixture of experts. This example subsumes and improves upon the recent results from \cite{LuoSch15, chaudhuri2009parameter} and provides an exact analogue to the PAC Bayesian theorem from statistical learning. Quantile expert bounds can be easily recovered from the result.

\begin{example}[\textbf{Generalized Predictable Sequences (Supervised Learning)}]
Consider an online supervised learning problem with a convex $1$-Lipschitz loss. Let $(M_t)_{t\ge1}$ be any predictable sequence that the learner can compute at round $t$ based on information provided so far, including $x_t$ (One can think of the predictable sequence $M_t$ as a prior guess for the hypothesis we would compare with in hindsight). Then the following adaptive rate is achievable:
\begin{small}
\begin{align*}
\B_n(f;x_{1:n}) =  \inf_{\gamma}\Biggl\{ K_1 \sqrt{\log n \cdot \log \cN_{2}(\F,\gamma/2, n) \cdot  \left(\sum_{t=1}^n \left(f(x_t) - M_t\right)^2+1\right) }&
\\ + K_2 \log n\int_{1/n}^\gamma \sqrt{n \log \cN_{2}(\F,\delta, n)} d \delta  + 2 \log n + &7\Biggr\},
\end{align*}
\end{small}for constants $K_1=4\sqrt{2},K_2=24\sqrt{2}$ from Corollary \ref{cor:small_loss_supervised}. The achievability is a direct consequence of Eq. \eqref{eq:convlip} in Lemma~\ref{lem:sym}, followed by Corollary \ref{cor:small_loss_supervised} (one can include any predictable sequence in the Rademacher average part because $\sum_{t} M_t \epsilon_t$ is zero mean). Particularly, if we assume that the sequential covering of class $\F$ grows as $\log \cN_2(\F,\epsilon, n) \le \epsilon^{-p}$ for some $p < 2$, we get that
\begin{small}\[
\B_n(f) = \tilde{O}\left(  \left(\sqrt{\textstyle{}\sum_{t=1}^n \left(f(x_t) - M_t\right)^2 + 1 }\right)^{1 - \frac{p}{2}} \left(\sqrt{n} \right)^{p/2} \right).
\]
\end{small}As $p$ gets closer to $0$, we get full adaptivity and replace $n$ by $\sum_{t=1}^n \left(f(x_t) - M_t\right)^2 + 1$. On the other hand, as $p$ gets closer to $2$ (i.e.  more complex function classes), we do not adapt and get a uniform bound in terms of $n$. For $p\in(0,2)$, we attain a natural interpolation. 
\end{example}

\begin{example}[\textbf{Regret to Fixed Vs Regret to Best (Supervised Learning)}]
Consider an online supervised learning problem with a convex $1$-Lipschitz loss and let $\abs{\F} = N$. Let $f^{\star}\in{}\F$ be a fixed expert chosen in advance. The following bound is achievable:
\begin{small}\[
\B_n(f, \xr[n]) = 4\log\prn*{\log{}N\sum_{t=1}^{n}(f(x_t) - f^{\star}(x_t))^{2} + e}\sqrt{32\prn*{\log{}N\sum_{t=1}^{n}(f(x_t) - f^{\star}(x_t))^{2} + e}} + 2.
\]
\end{small}In particular, against $f^{\star}$ we have $\B_n(f^{\star}, \xr[n]) = O(1)$, and against an arbitrary expert we have $\B_n(f, \xr[n]) = O\prn*{\sqrt{n\log{}N}\prn*{\log{}(n \cdot\log{}N})}$.
This bound follows from Eq. \eqref{eq:convlip} in Lemma~\ref{lem:sym} followed by Corollary \ref{cor:small_loss_finite}. 
This extends the study of \cite{even2008regret} to supervised learning and general class of experts $\F$.
\end{example}

\begin{example}[\textbf{Optimistic PAC-Bayes}] \label{eg:pacb}
Assume that we have a countable set of experts and that the loss for each expert on any round is non-negative and bounded by $1$. The function class $\F$ is the set of all distributions over these experts, and $\X = \{0\}$. This setting can be formulated as online linear optimization where the loss of mixture $f$ over experts, given instance $y$, is $\ip{f}{y}$, the expected loss under the mixture. The following adaptive bound is achievable:
\begin{small}
\begin{align*}
\B_{n}(f; y_{1:n}) &=  \sqrt{50 \left(\mrm{KL}(f|\pi)+\log(n)\right) \sum_{t=1}^n \En_{i\sim{}f}\ip{e_i}{y_t}^2} + 50 \left(\mrm{KL}(f|\pi) + \log(n)\right)   + 10.
\end{align*}
\end{small}This adaptive bound is an \textbf{online PAC-Bayesian bound}. The rate adapts not only to the KL divergence of $f$ with fixed prior $\pi$ but also replaces $n$ with $\sum_{t=1}^n \En_{i\sim{}f}\ip{e_i}{y_t}^2$. Note that we have $\sum_{t=1}^n \En_{i\sim{}f}\ip{e_i}{y_t}^2 \leq \sum_{t=1}^n \ip{f}{y_t}$, yielding the small-loss type bound described earlier. This is an improvement over the bound in \cite{LuoSch15} in that the bound is independent of number of experts, and so holds even for countably infinite sets of experts. The KL term in our bound may be compared to the MDL-style term in the bound of \cite{Koolen15}. If we have a large (but finite) number of experts and take $\pi$ to be uniform, the above bound provides an improvement over both \cite{chaudhuri2009parameter}\footnote{See \cite{LuoSch15} for a comparison of $\KL$-based bounds and quantile bounds.} and \cite{LuoSch15}. \\[1mm]
Evaluating the above bound with a distribution $f$ that places all its weight on any one expert appears to address the open question posed by \cite{cesa2007improved} of obtaining algorithm-independent oracle-type variance bounds for experts.
The proof of achievability of the above rate is shown in the appendix because it requires a slight variation on the symmetrization lemma specific to the problem.
\end{example}

% !TEX root = paper.tex

\section{Relaxations for Adaptive Learning}
To design algorithms for achievable rates, we extend the framework of online relaxations from \cite{RSS12}. A relaxation $\mathbf{Rel}_n : \bigcup_{t =0}^n \X^{t} \times \Y^{t} \to \reals$ that satisfies the \emph{initial condition},
\begin{small}\begin{equation}
\label{eq:initial}
\Relaxn(\xr[n], \yr[n]) \geq -\inf_{f\in\F}\crl*{
\sum\limits_{t=1}^{n}\ls(f(x_t), y_t) + \B_n(f;\xr[n],\yr[n])
},
\end{equation}
\end{small}and the \emph{recursive condition},
\begin{small}\begin{equation}
\label{eq:admissibility}
\Relaxn(\xr[t-1], \yr[t-1]) \geq
\sup_{x_t\in{}\X}\inf_{q_t\in{}\Delta(\D)}\sup_{y_t\in\Y}\En_{\hat{y}\sim{}q_t}\brk*{
\ls(\hat{y}_t,y_t) + \Relaxn(\xr[t], \yr[t])
},
\end{equation}
\end{small}is said to be \emph{admissible for the adaptive rate $\B_n$}.
The relaxation's corresponding strategy is \linebreak $\hat{q}_t = \argmin_{q_t\in{}\Delta(\D)}\sup_{y_t\in\Y}\En_{\hat{y}\sim{}q_t}\brk*{
\ls(\hat{y}_t,y_t) + \Relaxn(\xr[t], \yr[t])
 }$,
which enjoys the adaptive bound
\begin{small}
\[
\sum\limits_{t=1}^{n}\ls(\hat{y}_t, y_t)
-\inf_{f\in\F}\crl*{
\sum\limits_{t=1}^{n}\ls(f(x_t), y_t) + \B_n(f;\xr[n],\yr[n])} \leq \Relaxn(\cdot)\quad\forall{}\xr[n],\yr[n].
\]
\end{small}It follows immediately that the strategy achieves the rate $\B_{n}(f;\xr[n], \yr[n]) + \Relaxn(\cdot)$. Our goal is then to find relaxations for which the strategy is computationally tractable and $\Relaxn(\cdot)\leq{}0$ or at least has smaller order than $\B_n$. Similar to \cite{RSS12}, conditional versions of the offset minimax values $\A_n$ yield admissible relaxations, but solving these relaxations may not be computationally tractable. 

\begin{example}[\textbf{Online PAC-Bayes}]
\label{eg:kl_alg}
Consider the experts setting in Example \ref{eg:pacb} with:
\[
\B_{n}(f) = 3\sqrt{2n\max\crl*{\KL(f\mid{}\pi),1}} + 4\sqrt{n}.
\]
Let $R_{i}=2^{i-1}$ and let $q_{t}^{R}(y)$ denote the exponential weights distribution with learning rate $\sqrt{R/n}$: $
q^{R}(\yr[t])_{k} \propto \pi_{k}\exp\prn*{-\sqrt{R/n}(\sum_{s=1}^{t}y_t)_{k} }$.
The following is an admissible relaxation achieving $\B_n$:
\begin{small}
\[
\Bay(\yr[t]) = \inf_{\lambda>0}\brk*{
\frac{1}{\lambda}\log\prn*{\sum_{i}\exp\prn*{
    -\lambda\brk*{\sum_{s=1}^{t}\tri*{q^{R_i}(\yr[s-1]),y_s} +\sqrt{nR_{i}}
}
}
}
+ 2\lambda(n-t)
}.
\]
\end{small}
Let $q_{t}^{\star}$ be a distribution with
$
(q_{t}^{\star})_{i}\propto \exp\prn*{-\frac{1}{\sqrt{n}}\brk*{\sum_{s=1}^{t-1}\tri*{q^{R_i}(\yr[s-1]),y_s} -
      \sqrt{nR_{i}}
}}.
$ We predict by drawing $i$ according to $q_{t}^{\star}$, then drawing an expert according to $q^{R_i}(\yr[t-1])$.
%A proof for this algorithm is given in Appendix \ref{app:algo}.
% This algorithm can be interpreted as running a ``low-level'' instance of the exponential weights algorithm for each complexity radius $R_{i}$, then combining the predictions of these algorithms with a ``high-level'' instance. The high-level distribution $q_t^{\star}$ differs slightly from the usual exponential weights distribution in that it incorporates a prior whose weight decreases as the complexity radius increases. The prior distribution prevents the strategy from incurring a penalty that depends on the range of values the complexity radii take on, which would happen if the standard exponential weights distribution were used. 
\end{example}
While in general the problem of obtaining an efficient adaptive relaxation might be hard, one can ask the question,  ``If and efficient relaxation $\Relaxn^{R}$ is available for each $\F(R)$, can one obtain an adaptive model selection algorithm for all of $\F$?''. To this end for supervised learning problem with convex Lipschitz loss we delineate a meta approach which utilizes existing relaxations for each $\F(R)$. 
%to obtain algorithm for general adaptation.

\begin{lemma} 
\label{lemma:ada}
Let $q_{t}^{R}(y_1,\ldots,y_{t-1})$ be the randomized strategy corresponding to  $\Relaxn^{R}$, obtained after observing outcomes $y_1,\ldots,y_{t-1}$, and let $\theta:\R\to{}\R$ be nonnegative. The following relaxation is admissible for the rate $\B_n(R) = \Relaxn^{R}(\cdot)\theta(\Relaxn^{R}(\cdot))$:
{\small 
\begin{align*}
&\Ada(\xr[t], \yr[t]) =\\ &\sup_{\x, \y, \y'} \En_{\eps_{t+1:n}}\sup_{R\ge 1}\brk*{
\Relaxn^{R}(\xr[t], \yr[t]) - \Relaxn^{R}(\cdot)\theta(\Relaxn^{R}(\cdot)) + 2\sum\limits_{s=t+1}^{n} \eps_{s}\En_{\hat{y}_{s}\sim{}q_{s}^{R}(\yr[t],\y'_{t+1:s-1}(\eps))}\ls(\hat{y}_{s}, \y_{s}(\eps))
}.
\end{align*}}
\end{lemma}
Playing according to the strategy for $\Ada$ will guarantee a regret bound of $\B_n(R)+\Ada(\cdot)$, and $\Ada(\cdot)$ can be bounded using Proposition \ref{prop:main} when the form of $\theta$ is as in that proposition.

We remark that the above strategy is not necessarily obtained by running a high-level experts algorithm over the discretized values of $R$. It is an interesting question to determine the cases when such a strategy is optimal. More generally, when the adaptive rate $\B_n$ depends on data, it is not possible to obtain the rates we show non-constructively in this paper using the exponential weights algorithm with meta-experts as the required weighting over experts would be data dependent (and hence is not a prior over experts). Further, the bounds from exponential-weights-type algorithms are akin to having sub-exponential tails in Proposition \ref{prop:main}, but for many problems we may have sub-gaussian tails.  

Obtaining computationally efficient methods from the proposed framework is an interesting research direction. Proposition \ref{prop:main} provides a useful non-constructive tool to establish achievable adaptive bounds, and a natural question to ask is if one can obtain a constructive counterpart for the proposition.

\subsubsection*{References}
\begingroup
\renewcommand{\section}[2]{}%Get rid of automatic title.
\bibliographystyle{unsrt}
\small{
\bibliography{refs}

\begin{thebibliography}{10}

\bibitem{birge1998minimum}
Lucien Birg{\'e}, Pascal Massart, et~al.
\newblock Minimum contrast estimators on sieves: exponential bounds and rates
  of convergence.
\newblock {\em Bernoulli}, 4(3):329--375, 1998.

\bibitem{lugosi1999adaptive}
G{\'a}bor Lugosi and Andrew~B Nobel.
\newblock Adaptive model selection using empirical complexities.
\newblock {\em Annals of Statistics}, pages 1830--1864, 1999.

\bibitem{bartlett2002model}
Peter~L. Bartlett, St{\'e}phane Boucheron, and G{\'a}bor Lugosi.
\newblock Model selection and error estimation.
\newblock {\em Machine Learning}, 48(1-3):85--113, 2002.

\bibitem{massart2007concentration}
Pascal Massart.
\newblock {\em Concentration inequalities and model selection}, volume~10.
\newblock Springer, 2007.

\bibitem{Mendelson14}
Shahar {Mendelson}.
\newblock {Learning without Concentration}.
\newblock In {\em Conference on Learning Theory}, 2014.

\bibitem{offset2015}
Tengyuan Liang, Alexander Rakhlin, and Karthik Sridharan.
\newblock Learning with square loss: Localization through offset rademacher
  complexity.
\newblock {\em Proceedings of The 28th Conference on Learning Theory}, 2015.

\bibitem{onr2014}
Alexander Rakhlin and Karthik Sridharan.
\newblock Online nonparametric regression.
\newblock {\em Proceedings of The 27th Conference on Learning Theory}, 2014.

\bibitem{RST10}
Alexander Rakhlin, Karthik Sridharan, and Ambuj Tewari.
\newblock Online learning: Random averages, combinatorial parameters, and
  learnability.
\newblock In {\em Advances in Neural Information Processing Systems 23}. 2010.

\bibitem{hazan2010extracting}
Elad Hazan and Satyen Kale.
\newblock Extracting certainty from uncertainty: Regret bounded by variation in
  costs.
\newblock {\em Machine learning}, 80(2):165--188, 2010.

\bibitem{Chiangetal12}
Chao-Kai Chiang, Tianbao Yang, Chia-Jung Lee, Mehrdad Mahdavi, Chi-Jen Lu, Rong
  Jin, and Shenghuo Zhu.
\newblock Online optimization with gradual variations.
\newblock In {\em Conference on Learning Theory}, 2012.

\bibitem{RakSri13pred}
Alexander Rakhlin and Karthik Sridharan.
\newblock Online learning with predictable sequences.
\newblock In {\em Proceedings of the 26th Annual Conference on Learning Theory
  (COLT)}, 2013.

\bibitem{duchi2011adaptive}
John Duchi, Elad Hazan, and Yoram Singer.
\newblock Adaptive subgradient methods for online learning and stochastic
  optimization.
\newblock {\em The Journal of Machine Learning Research}, 12:2121--2159, 2011.

\bibitem{cesa2007improved}
Nicolo Cesa-Bianchi, Yishay Mansour, and Gilles Stoltz.
\newblock Improved second-order bounds for prediction with expert advice.
\newblock {\em Machine Learning}, 66(2-3):321--352, 2007.

\bibitem{chaudhuri2009parameter}
Kamalika Chaudhuri, Yoav Freund, and Daniel~J Hsu.
\newblock A parameter-free hedging algorithm.
\newblock In {\em Advances in neural information processing systems}, pages
  297--305, 2009.

\bibitem{McMahan2014}
H.~Brendan McMahan and Francesco Orabona.
\newblock Unconstrained online linear learning in hilbert spaces: Minimax
  algorithms and normal approximations.
\newblock {\em Proceedings of The 27th Conference on Learning Theory}, 2014.

\bibitem{PLG}
Nicolo Cesa-Bianchi and G{\'a}bor Lugosi.
\newblock {\em Prediction, Learning, and Games}.
\newblock Cambridge University Press, 2006.

\bibitem{srebro2010smoothness}
Nathan Srebro, Karthik Sridharan, and Ambuj Tewari.
\newblock Smoothness, low noise and fast rates.
\newblock In {\em Advances in neural information processing systems}, pages
  2199--2207, 2010.

\bibitem{LuoSch15}
Haipeng Luo and Robert~E Schapire.
\newblock Achieving all with no parameters: Adanormalhedge.
\newblock In {\em Conference on Learning Theory}, pages 1286--1304, 2015.

\bibitem{Koolen15}
Wouter~M. Koolen and Tim van Erven.
\newblock Second-order quantile methods for experts and combinatorial games.
\newblock In {\em Proceedings of the 28th Annual Conference on Learning Theory
  (COLT)}, pages 1155--1175, 2015.

\bibitem{Cover67}
Thomas~M. Cover.
\newblock Behavior of sequential predictors of binary sequences.
\newblock In {\em Proc. 4th Prague Conference on Information Theory,
  Statistical Decision Functions, Random Processes}, pages 263--272. Publishing
  House of the Czechoslovak Academy of Sciences, 1967.

\bibitem{StatNotes2012}
Alexander Rakhlin and Karthik Sridharan.
\newblock Statistical learning theory and sequential prediction, 2012.
\newblock Available at {\small
  \url{http://stat.wharton.upenn.edu/~rakhlin/book_draft.pdf}}.

\bibitem{Pinelis}
Iosif Pinelis.
\newblock Optimum bounds for the distributions of martingales in banach spaces.
\newblock {\em The Annals of Probability}, 22(4):1679--1706, 10 1994.

\bibitem{RakSriTew10beyond_arxiv}
Alexander Rakhlin, Karthik Sridharan, and Ambuj Tewari.
\newblock Online learning: Beyond regret.
\newblock {\em Journal of Machine Learning Research - Proceedings Track},
  19:559--594, 2011.

\bibitem{RakSriTew14ptrf}
Alexander Rakhlin, Karthik Sridharan, and Ambuj Tewari.
\newblock Sequential complexities and uniform martingale laws of large numbers.
\newblock {\em Probability Theory and Related Fields}, 2014.

\bibitem{even2008regret}
Eyal Even-Dar, Michael Kearns, Yishay Mansour, and Jennifer Wortman.
\newblock Regret to the best vs. regret to the average.
\newblock {\em Machine Learning}, 72(1-2):21--37, 2008.

\bibitem{RSS12}
Alexander Rakhlin, Ohad Shamir, and Karthik Sridharan.
\newblock Relax and randomize: From value to algorithms.
\newblock {\em Advances in Neural Information Processing Systems 25}, pages
  2150--2158, 2012.

\end{thebibliography}
}
\endgroup

\newpage
\appendix
% !TEX root = paper.tex
\section{Adaptive Rates and Achievability}
\begin{small}
\begin{proof}[\textbf{Proof of Lemma~\ref{lem:sym}}]
We first prove Eq. \eqref{eq:general}. We start from the definition of $\A_n(\F)$. Our proof proceeds ``inside out" by starting with the $n^{th}$ term and then working backwards by repeatedly applying the minimax theorem. To this end on similar lines as in \cite{RakSriTew14ptrf,onr2014,StatNotes2012}, we start with the inner most term as,
\begin{small}
  \begin{align}
    \sup_{x_n\in{}\X}&\inf_{q_n\in{}\Delta(\D)}\sup_{y_n\in{}\Y} \left(\Es{\hat{y}_n \sim q_n}{\ls(\hat{y}_n, y_n) -\inf_{f\in{}\F}\left\{\sum_{t=1}^{n}\ls(f(x_t), y_t) + \B_n(f ; \xr[n], \yr[n])\right\}}\right) \notag\\
    &= \sup_{x_n\in{}\X}\inf_{q_n\in{}\Delta(\D)}\sup_{p_n\in{}\Delta(\Y)} \left(\Es{\substack{\hat{y}_n \sim q_n\\ y_n \sim p_n}}{\sum_{t=1}^{n}\ls(\hat{y}_t, y_t) -\inf_{f\in{}\F}\left\{\sum_{t=1}^{n}\ls(f(x_t), y_t) + \B_n(f ; \xr[n], \yr[n])\right\}}\right)  \notag\\
    &= \sup_{x_n\in{}\X}\sup_{p_n\in{}\Delta(\Y)} \inf_{q_n\in{}\Delta(\D)}\left(\Es{\substack{\hat{y}_n \sim q_n\\ y_n \sim p_n}}{\sum_{t=1}^{n}\ls(\hat{y}_t, y_t) -\inf_{f\in{}\F}\left\{\sum_{t=1}^{n}\ls(f(x_t), y_t) + \B_n(f ; \xr[n], \yr[n])\right\}}\right) \notag\\
    &= \sup_{x_n\in{}\X}\sup_{p_n\in{}\Delta(\Y)} \inf_{\hat{y}_n \in{} \D}\left(\Es{y_n \sim p_n}{\sum_{t=1}^{n}\ls(\hat{y}_t, y_t) -\inf_{f\in{}\F}\left\{\sum_{t=1}^{n}\ls(f(x_t), y_t) + \B_n(f ; \xr[n], \yr[n])\right\}}\right) \notag\\
    &= \sup_{x_n\in{}\X}\sup_{p_n\in{}\Delta(\Y)} \left( \Es{y_n \sim p_n}{\sup_{f\in{}\F}\left\{ \inf_{\hat{y}_n \in{} \D} \Es{y_n \sim p_n}{\sum_{t=1}^{n}\ls(\hat{y}_t, y_t)} - \sum_{t=1}^{n}\ls(f(x_t), y_t) - \B_n(f ; \xr[n], \yr[n])\right\}}\right)\notag.
  \end{align}
\end{small}To apply the minimax theorem in step 3 above, we note that the term in the round bracket is linear in $q_n$ and in $p_n$ (as it is an expectation). Hence under mild assumptions on the sets $\D$ and $\Y$, the losses, and the adaptive rate $\B_n$, one can apply a generalized version of the minimax theorem to swap $\sup_{p_n}$ and $\inf_{q_n}$. Compactness of the sets and lower semi-continuity of the losses and $\B_n$ are sufficient, but see \cite{RakSriTew14ptrf,StatNotes2012} for milder conditions. Proceeding backward from $n$ to $1$ in a similar fashion we end up with the following quantity:
\begin{small}
  \begin{align}
    &\A_{n}(\F) \notag\\
    &= \dtri*{\sup_{x_t\in{}\X}\inf_{q_t\in{}\Delta(\D)}\sup_{y_t\in{}\Y}\underset{\hat{y}_t\sim{}q_t}{\En}}_{t=1}^{n}
                 \brk*{
                 \sum_{t=1}^{n}\ls(\hat{y}_t, y_t) -\inf_{f\in{}\F}\left\{\sum_{t=1}^{n}\ls(f(x_t), y_t) + \B_n(f ; \xr[n], \yr[n])\right\}
                 } \notag\\
               &= \dtri*{\sup_{x_t\in{}\X}\sup_{p_t\in{}\Delta(\Y)}\underset{y_t\sim{}p_t}{\En}}_{t=1}^{n}
                 \brk*{
                 \sup_{f\in{}\F}\left\{ \sum_{t=1}^{n} \inf_{\hat{y}_t \in \D} \Es{y_t \sim p_t}{\ls(\hat{y}_t, y_t)} -\sum_{t=1}^{n}\ls(f(x_t), y_t) - \B_n(f ; \xr[n], \yr[n])\right\}}\notag\\
               &\le \dtri*{\sup_{x_t\in{}\X}\sup_{p_t\in{}\Delta(\Y)}\underset{y_t\sim{}p_t}{\En}}_{t=1}^{n}
                 \brk*{ \sup_{f\in{}\F}\left\{ \sum_{t=1}^{n} \Es{y'_t \sim p_t}{\ls(f(x_t), y'_t)} - \ls(f(x_t), y_t) - \B_n(f ; \xr[n], \yr[n])\right\}}.\label{eq:inter}
  \end{align}
\end{small}See \cite{StatNotes2012} for more details of the steps involved in obtaining the above equality. To proceed, we use Jensen's inequality to pull out the expectations w.r.t. $y'_t$s, which gives
\begin{small}
  \begin{align*}
    & \le \dtri*{\sup_{x_t\in{}\X}\sup_{p_t\in{}\Delta(\Y)}\underset{y_t, y'_t \sim{}p_t}{\En}}_{t=1}^{n}
      \brk*{ \sup_{f\in{}\F}\left\{ \sum_{t=1}^{n} \ls(f(x_t), y'_t) - \ls(f(x_t), y_t) -  \B_n(f; \xr[n], \yr[n]) \right\} } \\
    & \le \dtri*{\sup_{x_t\in{}\X}\sup_{p_t\in{}\Delta(\Y)}\underset{y_t, y'_t \sim{}p_t}{\En} \sup_{y''_t \in \Y}}_{t=1}^{n}
      \brk*{ \sup_{f\in{}\F}\left\{ \sum_{t=1}^{n} \ls(f(x_t), y'_t) - \ls(f(x_t), y_t)-  \B_n(f; \xr[n], \yr[n]'')  \right\} } \\
    & = \dtri*{\sup_{x_t\in{}\X}\sup_{p_t\in{}\Delta(\Y)}\underset{y_t, y'_t \sim{}p_t}{\En} \En_{\epsilon_t} \sup_{y''_t \in \Y}}_{t=1}^{n}
      \brk*{ \sup_{f\in{}\F}\left\{ \sum_{t=1}^{n} \epsilon_t \left(\ls(f(x_t), y'_t) - \ls(f(x_t), y_t)\right) -  \B_n(f; \xr[n], \yr[n]'') \right\}  } \\
    & \le \dtri*{\sup_{x_t\in{}\X}\sup_{y_t, y'_t \in \Y} \En_{\epsilon_t} \sup_{y''_t \in \Y}}_{t=1}^{n}
      \brk*{ \sup_{f\in{}\F}\left\{ \sum_{t=1}^{n} \epsilon_t \left(\ls(f(x_t), y'_t) - \ls(f(x_t), y_t)\right) -  \B_n(f; \xr[n], \yr[n]'') \right\}  } \\
    & \le \dtri*{\sup_{x_t\in{}\X}\sup_{y_t \in \Y} \En_{\epsilon_t} \sup_{y''_t \in \Y}}_{t=1}^{n}
      \brk*{ \sup_{f\in{}\F}\left\{ \sum_{t=1}^{n} 2 \epsilon_t \ls(f(x_t), y_t) -  \B_n(f; \xr[n], \yr[n]'') \right\}  } \\
    & = \sup_{\x , \y, \y'}\Es{\epsilon}{ \sup_{f\in{}\F}\left\{ 2 \sum_{t=1}^{n} \epsilon_t \ls(f(\x_t(\epsilon)), \y_t(\epsilon)) - \B_n(f; \x_{1:n}(\epsilon),\y'_{2:n+1}(\epsilon)) \right\} },
  \end{align*}
\end{small}where in the last step we switch to tree notation, but keep in mind that each $y''_t$ is picked after drawing $\epsilon_t$, and thus the tree $\y'$ appears with one index shifted.\\

We now proceed to prove inequality \eqref{eq:convlip}. Here, we employ the convexity assumption $\ell(\hat{y}_t,y_t) - \ell(f(x_t),y_t) \le \ell'(\hat{y}_t,y_t) (\hat{y}_t - f(x_t))$, where the derivative is with respect to the first argument. As before, applying the minimax theorem, 
\begin{small}
  \begin{align*}
    \A_{n}(\F) &= \dtri*{\sup_{x_t\in{}\X}\inf_{q_t\in{}\Delta(\D)}\sup_{y_t\in{}\Y}\underset{\hat{y}_t\sim{}q_t}{\En}}_{t=1}^{n}
                 \brk*{\sum_{t=1}^{n}\ls(\hat{y}_t, y_t) -\inf_{f\in{}\F}\left\{\sum_{t=1}^{n}\ls(f(x_t), y_t) + \B_n(f ; \xr[n], \yr[n])\right\}}\\ 
               & = \dtri*{\sup_{x_t\in{}\X}\sup_{p_t\in{}\Delta(\Y)} \inf_{\hat{y}_t\in{}\D} \underset{y_t\sim{}p_t}{\En}}_{t=1}^{n}
                 \brk*{\sum_{t=1}^{n}\ls(\hat{y}_t, y_t) -\inf_{f\in{}\F}\left\{\sum_{t=1}^{n}\ls(f(x_t), y_t) + \B_n(f ; \xr[n], \yr[n])\right\}}\\
               & \le \dtri*{\sup_{x_t\in{}\X}\sup_{p_t\in{}\Delta(\Y)} \inf_{\hat{y}_t\in{}\D} \underset{y_t\sim{}p_t}{\En}}_{t=1}^{n}
                 \brk*{\sup_{f\in{}\F}\left\{ \sum_{t=1}^{n} \ls'(\hat{y}_t, y_t)(\hat{y}_t - f(x_t)) - \B_n(f ; \xr[n], \yr[n])\right\}}.
  \end{align*}
\end{small}We may now pick $\hat{y}_t = \hat{y}^*_t(p_t) \defeq \argmin_{\hat{y}} \Es{y_t \sim p_t}{\ell(\hat{y}_t,y_t)}$. By convexity (and assuming the loss allows swapping of derivative and expectation), $\Es{y_t \sim p_t}{\ls'(\hat{y}_t, y_t)} = 0$.
This (sub)optimal strategy yields an upper bound of
\begin{small}
  \begin{align*}
    &\dtri*{\sup_{x_t\in{}\X}\sup_{p_t\in{}\Delta(\Y)}  \underset{y_t\sim{}p_t}{\En}}_{t=1}^{n}
      \brk*{\sup_{f\in{}\F}\left\{ \sum_{t=1}^{n} \left(\ls'(\hat{y}^*_t, y_t) - \Es{y'_t \sim p_t}{\ls'(\hat{y}^*_t, y'_t)} \right)(\hat{y}^*_t - f(x_t)) - \B_n(f ; \xr[n], \yr[n])\right\}}.
  \end{align*}
\end{small}Since $\left(\ls'(\hat{y}^*_t, y_t) - \Es{y'_t \sim p_t}{\ls'(\hat{y}^*_t, y'_t)} \right) \hat{y}^*_t$ is independent of $f$ and has expected value of $0$, the above quantity is equal to
\begin{small}
  \begin{align*}
    &\dtri*{\sup_{x_t\in{}\X}\sup_{p_t\in{}\Delta(\Y)}  \underset{y_t\sim{}p_t}{\En}}_{t=1}^{n}
      \brk*{\sup_{f\in{}\F}\left\{ \sum_{t=1}^{n} \left( \Es{y'_t \sim p_t}{\ls'(\hat{y}^*_t, y'_t)} - \ls'(\hat{y}^*_t, y_t) \right)f(x_t) - \B_n(f ; \xr[n], \yr[n])\right\}}\\
    & \le \dtri*{\sup_{x_t\in{}\X}\sup_{p_t\in{}\Delta(\Y)}  \underset{y_t, y'_t\sim{}p_t}{\En}}_{t=1}^{n}
      \brk*{\sup_{f\in{}\F}\left\{ \sum_{t=1}^{n} \left(\ls'(\hat{y}^*_t, y'_t) -  \ls'(\hat{y}^*_t, y_t) \right)f(x_t) - \B_n(f ; \xr[n], \yr[n])\right\}}\\
    & = \dtri*{\sup_{x_t\in{}\X}\sup_{p_t\in{}\Delta(\Y)}  \underset{y_t, y'_t\sim{}p_t}{\En} \En_{\epsilon_t}}_{t=1}^{n}
      \brk*{\sup_{f\in{}\F}\left\{ \sum_{t=1}^{n} \epsilon_t \left(\ls'(\hat{y}^*_t, y'_t) - \ls'(\hat{y}^*_t, y_t) \right)f(x_t) - \B_n(f ; \xr[n], \yr[n])\right\}}.
  \end{align*}
\end{small}Replacing $\left(\ls'(\hat{y}^*_t, y'_t) - \ls'(\hat{y}^*_t, y_t) \right)$ by $2Ls_t$ for $s_t \in [-1,1]$ and taking supremum over $s_t$ we get,
\begin{small}
  \begin{align*}
    & \le \dtri*{\sup_{x_t\in{}\X}\sup_{p_t\in{}\Delta(\Y)}  \underset{y_t, y'_t\sim{}p_t}{\En} \sup_{s_t \in [-1,1]}\En_{\epsilon_t}}_{t=1}^{n}
      \brk*{\sup_{f\in{}\F}\left\{ \sum_{t=1}^{n} 2 L \epsilon_t s_t f(x_t) - \B_n(f ; \xr[n], \yr[n])\right\}}\\
    & \le  \dtri*{\sup_{x_t\in{}\X} \sup_{y_t} \sup_{s_t \in [-1,1]}\En_{\epsilon_t}}_{t=1}^{n}
      \brk*{\sup_{f\in{}\F}\left\{ \sum_{t=1}^{n} 2 L \epsilon_t s_t f(x_t) - \B_n(f ; \xr[n], \yr[n])\right\}}.
  \end{align*}
\end{small}Since the suprema over $s_t$ are achieved at $\{\pm1\}$ by convexity, the last expression is equal to
\begin{small}
  \begin{align*}
    &\dtri*{\sup_{x_t\in{}\X} \sup_{y_t} \sup_{s_t \in \{-1,1\}}\En_{\epsilon_t}}_{t=1}^{n}
      \brk*{\sup_{f\in{}\F}\left\{ \sum_{t=1}^{n} 2 L \epsilon_t s_t f(x_t) - \B_n(f ; \xr[n], \yr[n])\right\}}\\
    & = \dtri*{\sup_{x_t\in{}\X} \sup_{y_t} \En_{\epsilon_t}}_{t=1}^{n}
      \brk*{\sup_{f\in{}\F}\left\{ \sum_{t=1}^{n} 2 L \epsilon_t f(x_t) - \B_n(f ; \xr[n], \yr[n])\right\}}\\
    & = \sup_{\x,\y}\Es{\epsilon}{\sup_{f\in{}\F}\left\{ \sum_{t=1}^{n} 2 L \epsilon_t f(\x_t(\epsilon)) - \B_n(f ; \x_{1:n}(\epsilon), \y_{1:n}(\epsilon))\right\}}.
  \end{align*}
\end{small}In the last but one step we removed $s_t$, since for any function $\Psi$, and any $s \in \{\pm1\}$,
 $\E{\Psi(s \epsilon)} = \tfrac{1}{2}\left(\Psi(s)  + \Psi(-s)\right) = \tfrac{1}{2}\left(\Psi(1)  + \Psi(-1)\right) =  \E{\Psi(\epsilon)}$.
\end{proof}

\begin{proof}[\textbf{Proof of Proposition~\ref{prop:main}}]
	Define $Z_i = \left[X_{i} - B_i \theta_i\right]_+$. As long as $\theta_i \ge 1$, for any strictly positive $\tau$ we have the tail behavior
$$
P(Z_i \ge t) = P(X_{i} - B_i \theta_i  \ge \tau)  \le C_1 \exp\left(- \frac{(B_i(\theta_i-1)+\tau)^2}{2 \sigma_i^2} \right) +  C_2 \exp\left(- (B_i(\theta_i-1)+\tau) s_i \right).
$$
Note that for any positive sequence $(\delta_i)_{i\in I}$ with $\delta=\sum_{i\in I}\delta_i$,
\begin{align*}
 \E{\sup_{i \in I}\{X_i - B_i \theta_i\}} &\le \E{\sup_{i \in I}Z_i} \le \sum_{i\in I} \E{Z_i} \le \delta + \sum_{i\in I} \int_{\delta_i}^\infty P(Z_i \ge \tau) d\tau.
 \end{align*}
The sum of the integrals above is equal to 
 \begin{align*}
  \sum_{i\in I}& \int_{\delta_i}^\infty  P(X_i-B_i\theta_i \ge \tau) d\tau \\
  &\le  C_1 \sum_{i\in I} \int_{0}^\infty  \exp\left(- \frac{\left(B_i(\theta_i -1) + \tau\right)^2}{2\sigma_i^2} \right) dt + C_2 \sum_{i \in I} \int_{0}^\infty  \exp\left(- \left(B_i(\theta_i -1) + \tau\right)s_i \right) d\tau\\
  & \le C_1 \sum_{i \in I} \exp\left(- \frac{1}{2}\left(\frac{B_i}{\sigma_i}\right)^2\left(\theta_i -1\right)^2\right) \int_{0}^\infty  e^{ - \frac{\tau^2}{2 \sigma_i^2}}d\tau + C_2 \sum_{i \in I} \exp\left(- B_i s_i \left(\theta_i -1\right)\right) \int_{0}^\infty  e^{ - \tau s_i}d\tau \\
  &\leq \sqrt{\frac{\pi}{2}} C_1 \sum_{i \in I} \sigma_i \exp\left(- \frac{1}{2}\left(\frac{B_i}{\sigma_i}\right)^2\left(\theta_i -1\right)^2\right)  +  C_2 \sum_{i \in I} s_i^{-1} \exp\left(- B_i s_i \left(\theta_i -1\right)\right) \\
  & \le  \frac{\pi^2 \sqrt{\pi}}{6\sqrt{2}} C_1  \bar{\sigma} + \frac{\pi^2}{6} C_2 (\bar{s})^{-1},
\end{align*}
where the last step is obtained by plugging in 
$$\theta_i  =  \max \left\{ \frac{\sigma_i}{B_i}\sqrt{2\log(\sigma_i/\bar{\sigma}) + 4 \log(i)}  ,  (B_i s_i)^{-1}\log\left(i^2 (\bar{s}/s_i)\right)\right\} +1$$ 
and using as an upper bound $\frac{\sigma_i}{B_i}\sqrt{2 \log(i^{2} \sigma_i/\bar{\sigma})} +1$ for $\theta_i$ in the sub-gaussian part and $(B_i s_i)^{-1}\log\left(i^2 \bar{s}/s_i\right)+ 1$ for $\theta_i$ in the sub-exponential part. Since $\delta$ can be chosen arbitrarily small, we may over-bound the above constant and obtain the result.
\end{proof}

\begin{proof}[\textbf{Proof of Lemma~\ref{lem:proboffset}}]

Fix $\gamma>0$. For $j\geq 0$, let $V_j$ be a minimal sequential cover of $\G$ on $\z$ at scale $\beta_j = 2^{-j}\gamma$ and with respect to empirical $\ell_2$ norm. Let $\v^j[g,\epsilon]$ be an element guaranteed to be $\beta_j$-close to $f$ at the $j$-th level, for the given $\epsilon$. Choose $N = \log_{2}(2 \gamma n)$, so that $\beta_N n \le 1$. Let us use the shorthand $\cN_2(\gamma) \deq \cN_2(\G,\gamma,\z)$.

For any $\epsilon\in\{\pm1\}^n$ and $g\in \G$, 
$$\sum_{t=1}^n \epsilon_t g(\z_t(\epsilon)) - 2 \alpha  g(\z_t(\epsilon))^2 $$
can be written as
\begin{align*}
&\sum_{t=1}^n \left(\epsilon_t (g(\z_t(\epsilon)) -\v_t^0[g,\epsilon](\epsilon))\right) + \sum_{t=1}^n\left( \epsilon_t \v_t^0[g,\epsilon](\epsilon) - 2 \alpha  g(\z_t(\epsilon))^2 \right) \\
& \le  \sum_{t=1}^n \left(\epsilon_t (g(\z_t(\epsilon)) -\v_t^0[g,\epsilon](\epsilon))\right) + \sum_{t=1}^n\left( \epsilon_t \v_t^0[g,\epsilon](\epsilon) - \alpha  \v_t^0[g,\epsilon](\epsilon)^2 \right) \\
& = \sum_{t=1}^n \left(\epsilon_t (g(\z_t(\epsilon)) -\v_t^N[g,\epsilon](\epsilon)\right)  + \sum_{t=1}^n \sum_{k=1}^N  \epsilon_t \left(\v_t^{k}[g,\epsilon](\epsilon)-\v_t^{k-1}[g,\epsilon](\epsilon)\right) \\
&~~~~~~~~~ + \sum_{t=1}^n\left( \epsilon_t \v_t^0[g,\epsilon](\epsilon) - \alpha  \v^0_t[g,\epsilon](\epsilon)^2 \right).
\end{align*}
By Cauchy-Schwartz, the first term is upper bounded by $n\beta_N\leq 1$. The second term above is upper bounded by  
\begin{align*}
& \sum_{k=1}^N \sum_{t=1}^n   \epsilon_t \left(\v_t^k[g,\epsilon](\epsilon)-\v_t^{k-1}[g,\epsilon](\epsilon)\right)  
 \le \sum_{k=1}^N \sup_{\w^k \in W_k}\sum_{t=1}^n   \epsilon_t \w_t^k(\epsilon),
\end{align*}
where $W_k$ is a set of differences of trees for levels $k$ and $k-1$ (see \cite[Proof of Theorem 3]{RakSriTew14ptrf}). Finally, the third term is controlled by 
\begin{align*}
 \sum_{t=1}^n\left( \epsilon_t \v_t^0[g,\epsilon](\epsilon) - \alpha  \v_t^0[g,\epsilon](\epsilon)^2 \right) 
& \le \sup_{\v \in V_0}\sum_{t=1}^n\left( \epsilon_t \v_t(\epsilon) - \alpha  \v^2_t(\epsilon) \right).
\end{align*}
The probability in the statement of the Lemma can now be upper bounded by
{\small \begin{align*}
&P\left( \sum_{k=1}^N \sup_{\w^k \in W_k}\sum_{t=1}^n   \epsilon_t \w_t^k(\epsilon)   + \sup_{\v \in V_0}\sum_{t=1}^n\left( \epsilon_t \v_t(\epsilon) - \alpha  \v^2_t(\epsilon) \right)  -  \frac{\log \cN_{2}(\gamma)}{\alpha}  - 12\sqrt{2} \int_{1/n}^{\gamma} \sqrt{n \log \cN_{2}(\delta)} d \delta > \tau\right).
\end{align*}}In view of
\[
\sqrt{72} \sum_{k=1}^N \beta_k \sqrt{n \log \cN_{2}(\beta_k)} \le  12\sqrt{2} \int_{1/n}^{\gamma} \sqrt{n \log \cN_{2}(\delta)} d\delta
\]
this probability can be further upper bounded by
\begin{small}
\begin{align*}
& P\left( \sum_{k=1}^N \sup_{\w^k \in W_k}\sum_{t=1}^n   \epsilon_t \w_t^k(\epsilon)   + \sup_{\v \in V_0}\sum_{t=1}^n\left( \epsilon_t \v_t(\epsilon) - \alpha  \v^2_t(\epsilon) \right)  - \frac{\log \cN_{2}(\gamma)}{\alpha}  - \sqrt{72} \sum_{k=1}^N \beta_k \sqrt{n \log \cN_{2}(\beta_k)} > \tau\right).
\end{align*}
\end{small}Define a distribution $p$ on $\{1,\ldots,N\}$ by
$ p_k = \frac{\beta_k  \sqrt{n \log \cN_{2}(\beta_k)}}{\sum_{k=1}^N \beta_j \sqrt{n \log \cN_{2}(\beta_j)}}.
$
Then the above probability can be upper bounded by
\begin{align*}
& P\left( \exists k \in [N] \textrm{ s.t. } \sup_{\w^k \in W_k}\sum_{t=1}^n  \epsilon_t \w_t^k(\epsilon)   -  \sqrt{72}\beta_k \sqrt{n \log \cN_{2}(\beta_k)} > \frac{\tau p_k}{2}   \right.\\
&\left.~~~~~~~~ \vee~~ \sup_{\v \in V_0}\sum_{t=1}^n\left( \epsilon_t \v_t(\epsilon) - \alpha  \v^2_t(\epsilon) \right)  - \frac{\log \cN_{2}(\gamma)}{\alpha} > \frac{\tau}{2}  \right)\\
&\le  \sum_{k=1}^N P\left(\sup_{\w^k \in W_k}\sum_{t=1}^n  \epsilon_t \w_t^k(\epsilon)   -  \sqrt{72} \beta_k \sqrt{n \log \cN_{2}(\beta_k)} > \frac{\tau p_k}{2} \right) \\
& + P\left(\sup_{\v \in V_0}\sum_{t=1}^n\left( \epsilon_t \v_t(\epsilon) - \alpha  \v^2_t(\epsilon) \right)  - \frac{\log \cN_{2}(\gamma)}{\alpha} > \frac{\tau}{2}  \right).
\end{align*}
The second term can be upper bounded using Chernoff method by 
\begin{align*}
	&\sum_{\v \in V_0} P\left(\sum_{t=1}^n\left( \epsilon_t \v_t(\epsilon) - \alpha  \v^2_t(\epsilon) \right)  - \frac{\log \cN_{2}(\gamma)}{\alpha} > \frac{\tau}{2}  \right)\\
	&\leq \cN_{2}(\gamma) \exp\left(- \frac{\alpha \tau}{2} - \log \cN_{2}(\gamma)\right)
	\leq \exp\left(- \frac{\alpha \tau}{2}\right)
\end{align*}
while the first sum of probabilities can be upper bounded by
\begin{align}
	\label{eq:interm1}
\sum_{k=1}^N \sum_{\w^k \in W_k} P\left(\sum_{t=1}^n  \epsilon_t \w_t^k(\epsilon)   - \sqrt{72} \beta_k \sqrt{n \log \cN_{2}(\beta_k)} > \frac{\tau \beta_k  \sqrt{n \log \cN_{2}(\beta_k)}}{2\sum_{k=1}^N \beta_k \sqrt{n \log \cN_{2}(\beta_k)}} \right).
\end{align}
For any $k$, the tail probability above is controlled by Hoeffding-Azuma inequality as
\begin{align*}
	&P\left(\sum_{t=1}^n  \epsilon_t \w_t^k(\epsilon) >   \beta_k \sqrt{n \log \cN_{2}(\beta_k)}\left(6\sqrt{2} +  \frac{\tau}{2\sum_{k=1}^N \beta_k \sqrt{n \log \cN_{2}(\beta_k)}} \right)^2 \right) \\
&\leq \exp\left(-  \frac{1}{18} \log \cN_{2}(\beta_k)\left(6\sqrt{2} +  \frac{\tau}{2\sum_{k=1}^N \beta_k \sqrt{n \log \cN_{2}(\beta_k)}} \right)^2 \right) \\
&\leq \exp\left( -4\log \cN_{2}(\beta_k)\right) \exp\left( -\frac{\tau^2}{18\left(2\sum_{k=1}^N \beta_k \sqrt{n \log \cN_{2}(\beta_k)}\right)^2}\right),
\end{align*}
because $\frac{1}{n}\sum_{t=1}^n \w_t^k(\epsilon)^2 \leq 3\beta_k^2$ for any $\epsilon$ by triangle inequality (see \cite{RakSriTew14ptrf}). Then the double sum in \eqref{eq:interm1} is upper bounded by
\begin{align*}
\Gamma \exp\left( -\frac{\tau^2}{18\left(2\sum_{k=1}^N \beta_k \sqrt{n \log \cN_{2}(\beta_k)}\right)^2}\right),
\end{align*}
where $\Gamma \geq \sum_{k=1}^N \cN_{2}(\beta_k)^{-2}$. 
This upper bound can be further relaxed to
\begin{align*}
 \Gamma \exp\left(-  \frac{\tau^2}{ 2\left(12\int_{1/n}^{\gamma} \sqrt{n \log \cN_{2}(\delta)} d\delta\right)^2} \right).
\end{align*}
Since $N=\log_{2}(2\gamma n)$, we may take 
$$\Gamma =\sum_{k=1}^{\log_{2}(2\gamma n)} \cN_{2}(\gamma 2^{-k})^{-2}.$$
\end{proof}
\begin{proof}[\textbf{Proof of Corollary~\ref{cor:small_loss_supervised}}]
Let $\cN_{2}(\gamma)\deq \cN_{2}(\G,\gamma,\z)$. Observe that
{\small $$2 \sqrt{2(\log n) \left(\log \cN_{2}(\gamma/2) \right)\left( \sum_{t=1}^n g^2(\z_t(\epsilon))+1\right) } = \inf_{\alpha}\left\{ \frac{(\log n) \left(\log \cN_{2}(\gamma/2) \right)}{\alpha} + 2\alpha \left( \sum_{t=1}^n g^2(\z_t(\epsilon))+1\right) \right\}$$}and, furthermore, the optimal $\alpha$ is
$$\sqrt{\frac{(\log n) \left(\log \cN_{2}(\gamma/2) \right)}{2(\sum_{t=1}^n g^2(\z_t(\epsilon))+1)}}$$
which is a number between $d_\ell = \sqrt{\frac{(\log n) \left(\log \cN_{2}(\gamma/2) \right)}{2(n+1)}}$ and $d_u = \sqrt{(\log n) \left(\log \cN_{2}(\gamma/2) \right)}$ 
 as long as $\cN_{2}(\gamma/2)>1$. With this we get
 \begin{small}
   \begin{align}
%     \label{eq:split1}
      &\begin{aligned}
      \sup_{\substack{g \in \G\\ \gamma\in [n^{-1}, 1]}} \biggr[\sum_{t=1}^n \epsilon_t g(\z_t(\epsilon))  - 4 \sqrt{2(\log n) \left(\log \cN_{2}(\gamma/2) \right)\left( \sum_{t=1}^n g^2(\z_t(\epsilon))+1\right) }& \\
     - 24\sqrt{2}\log n  \int_{1/n}^{\gamma} \sqrt{n \log \cN_{2}(\delta)} d \delta  + 2\log n &\biggl] 
     \end{aligned} \notag\\
      &\begin{aligned}
      \leq \sup_{\substack{g \in \G\\ \gamma\in [n^{-1},1] , \alpha \in[d_\ell,d_u]}}\biggr[ \sum_{t=1}^n \epsilon_t g(\z_t(\epsilon)) - \frac{ 2(\log n) \left(\log \cN_{2}(\gamma/2) \right)}{\alpha} - 4 \alpha \sum_{t=1}^n g^2(\z_t(\epsilon))& \\ 
      - 24\sqrt{2} \log n \int_{1/n}^{\gamma} \sqrt{n \log \cN_{2}(\delta)} d \delta  - 2\log n&\biggl].
      \end{aligned}\label{eq:split1}
   \end{align}
 \end{small}The case of $\gamma\in[1/n, 2/n)$ will be considered separately. Let us assume $\gamma\geq 2/n$.
We now discretize both $\alpha$ and $\gamma$ by defining $\alpha_i = 2^{-(i-1)} d_u$ and $\gamma_j = 2^{j} n^{-1}$, $i,j\geq 1$. We go to an upper bound by mapping each $\alpha$ to $\alpha_i$ or $\alpha_i/2$, depending on the direction of the sign. Similarly, we map $\gamma$ to either $\gamma_i$ or $2\gamma_i$. The upper bound becomes
\begin{small}
\begin{align*}
          &\max_{i,j} \sup_{g \in \G} \sum_{t=1}^n \left(\epsilon_t g(\z_t(\epsilon)) - 2 \alpha_i  g^2(\z_t(\epsilon)) \right)- (2\log n)\left( \frac{\log \cN_{2}(\gamma_j)}{\alpha_i}  + 12\sqrt{2} \int_{1/n}^{\gamma_j} \sqrt{n \log \cN_{2}(\delta)} d \delta  + 1\right).
\end{align*}
\end{small}Given the doubling nature of $\alpha_i$ and $\gamma_j$, the indices $i,j$ are upper bounded by $O(\log n)$. Now define a collection of random variables indexed by $(i,j)$
$$X_{i,j} = \sup_{g \in \G} \sum_{t=1}^n \epsilon_t g(\z_t(\epsilon)) - 2 \alpha_i  g^2(\z_t(\epsilon))  $$
and constants $$B_{i,j} = \frac{\log \cN_{2}(\gamma_j)}{\alpha_i}  + 12\sqrt{2} \int_{1/n}^{\gamma_j} \sqrt{n \log \cN_{2}(\delta)} d \delta  + 1.$$
Lemma \ref{lem:proboffset} establishes that 
$$P(X_{i,j}-B_{i,j} > \tau) \leq  \Gamma \exp\left(-  \frac{\tau^2}{ 2\sigma_j^2} \right) +  \exp\left(- \frac{\alpha_i \tau}{2}\right)$$
where $\sigma_j = 12\sqrt{2}\int_{\frac{1}{n}}^{\gamma_j}  \sqrt{n \log \cN_{2}(\delta)} d\delta$  and $\Gamma$ as specified in Lemma \ref{lem:proboffset}. Whenever $\delta$-entropy grows as $\delta^{-p}$, $\sigma_j \leq 12\sqrt{2}  \sqrt{n}$, ensuring $\log(\sigma_j/\sigma_1) \leq \log(n)$. Further, we can take $1\leq \Gamma\leq \log(2n)$.

Proposition \ref{prop:main} is used with a sequence of random variables, but we can easily put the pairs $(i,j)$ into a vector of size at most $\log_{2}(n)^2$. Observe that $s_i=\alpha_i/2$, $(B_{i,j}s_i)^{-1} \leq 2$, $\sigma_j/ B_{i,j}\leq 1$, $s_1/s_i\leq \sqrt{2(n+1)}$. Then, by taking $\bar{\sigma}=\min\{1/\Gamma,\sigma_1\}$ and $\bar{s}=s_1$,
	\begin{align*}
		\theta_{k_{i,j}} &=  \max\left\{ \frac{\sigma_j}{B_{i,j}}\sqrt{2\log(\sigma_j/\bar{\sigma}) + 4 \log(k_{i,j})}  ,  (B_{i,j} s_i)^{-1}\log\left(k_{i,j}^2 (\bar{s}/s_i)\right)\right\} +1\\
		&\leq \max\left\{ \sqrt{2\log(n) + 2\log (\log(2n)) + 4 \log(k_{i,j})}  ,  2\log\left(k_{i,j}^2 \sqrt{2(n+1)}\right)\right\} +1
	\end{align*}
	where $k_{i,j}=(\log n)\cdot (i-1)+j$. This choice of the multiplier ensures
	$$\En \max_{i,j} \left\{ X_{i,j} - \theta_{k_{i,j}} B_{i,j} \right\} \leq 3\Gamma\bar{\sigma} + 4 \alpha_1^{-1}\leq 7$$
and $\theta_{i,j}$ is shown to be upper bounded by $2\log n$. Hence
 {\small\begin{align*}
&\E{\sup_{g \in \G, \gamma} \sum_{t=1}^n \epsilon_t g(\z_t(\epsilon))- 4 \sqrt{2\log n \log \cN_{2}(\gamma/2) \left( \sum_{t=1}^n g^2(\z_t(\epsilon)) +1 \right) } - 24\sqrt{2} \log n\int_{1/n}^{\gamma} \sqrt{n \log \cN_{2}(\delta)} d \delta} \\ &\le 7 + 2\log n.
\end{align*}}Now, consider the case $\gamma\in[1/n, 2/n)$. We upper bound \eqref{eq:split1} by
\begin{align*}
          &\max_{i} \sup_{g \in \G} \sum_{t=1}^n \left(\epsilon_t g(\z_t(\epsilon)) - 2 \alpha_i  g^2(\z_t(\epsilon)) \right)- (2\log n)\left( \frac{\log \cN_{2}(1/n)}{\alpha_i}    + 1\right),
\end{align*}
which is controlled by setting $\gamma=1/n$ in Lemma \ref{lem:proboffset}. This case is completed by invoking Proposition \ref{prop:main} as before.
\end{proof}
\begin{proof}[\textbf{Proof of Corollary~\ref{cor:small_loss_finite}}]
Assume $N > e$ and let $C>0$. We first note that
\begin{small}
\begin{align*}
&\inf_{\alpha > 0}\crl*{
\frac{C\log{}\prn*{\frac{\sqrt{C}\log{}N}{\alpha}}\log{}N}{\alpha} + \alpha\prn*{\sum_{t=1}^{n}g^{2}(\z_t(\eps)) + \frac{e}{\log{}N}}
}\\ &\leq
2\log\prn*{\log{}N\sum_{t=1}^{n}g^2(\z(\eps)) + e}\sqrt{C\prn*{\log{}N\sum_{t=1}^{n}g^2(\z(\eps)) + e}}
\end{align*}\end{small}with the inequality obtained using $\alpha^{\star} = \sqrt{\frac{C\log{}N}{\sum_{t=1}^{n}g^{2}(\z(\eps)) + e/\log{}N}
}$,
which is a number between $d_{\ls}\defeq{}\sqrt{\frac{C\log{}N}{n + e/log{}N}}$ and $d_{u} \defeq \sqrt{\frac{C}{e}}\log{}N$. 
Subsequently,
\begin{align*}
&\sup_{g\in{}\G}\sum_{t=1}^{n}\eps_{t}g(\z_{t}(\eps)) - 2\log\prn*{\log{}N\sum_{t=1}^{n}g^2(\z(\eps)) + e}\sqrt{C\prn*{\log{}N\sum_{t=1}^{n}g^2(\z(\eps)) + e}}\\
&\leq \sup_{\substack{g\in{}\G\\\alpha\in{}\brk{d_{\ls}, d_{u}}}}\brk*{\sum_{t=1}^{n}\eps_{t}g(\z_{t}(\eps)) 
- \alpha\sum_{t=1}^{n}g^{2}(\z_t(\eps))
-\frac{C\log{}N}{\alpha}\log{}\prn*{\frac{\sqrt{C}\log{}N}{\alpha}}}.
\end{align*}
Let $L = \ceil*{\log_{2}\prn*{\sqrt{\frac{n\log{}N}{e}+1}}+1}$. We discretize the range of $\alpha$ by defining $\alpha_{i} = d_{u}2^{-(i-1)}$ for $i\in{}\brk{L}$. The following upper bound holds:
\[
\sup_{\substack{g\in{}\G\\i\in{}\brk{L}}}\brk*{\sum_{t=1}^{n}\eps_{t}g(\z_{t}(\eps)) 
- \frac{\alpha_{i}}{2}\sum_{t=1}^{n}g^{2}(\z_t(\eps))
-\frac{C\log{}N}{\alpha_i}\log{}\prn*{\frac{\sqrt{C}\log{}N}{\alpha_i}}}.
\]
Define a collection of random variables indexed by $i\in{}\brk{L}$ with
\[
X_{i} = \sup_{g\in{}\G}\brk*{\sum_{t=1}^{n}\eps_{t}g(\z_{t}(\eps)) 
- \frac{\alpha_{i}}{2}\sum_{t=1}^{n}g^{2}(\z_t(\eps))}
\]
and let $B_{i} = \frac{4\log{}N}{\alpha_{i}}$. Applying Lemma \ref{lem:proboffset} with $\gamma=1/n$ establishes
\[
P(X_{i} - B_{i} > \tau) \leq \exp\prn*{-\frac{\alpha_{i}\tau}{8}}.
\]
We now set $s_{i} = \alpha_i/8$ and $\bar{s} = s_{1}$, and apply Proposition \ref{prop:main}, yielding
\[
\En\crl*{X_{i} - B_{i}\theta_{i}} \leq \frac{16\sqrt{e}}{C}.
\]
It remains to relate this quantity to the rate we are trying to achieve. Note that our bound on $P(X_{i} - B_{i} > \tau)$ has a pure exponential tail, so we only need to consider $\theta_{i} = (B_{i}s_{i})^{-1}\log(i^{2}(\bar{s}/s_{i}))+1$. Taking $C\geq{}32$ and observing that $(B_{i}s_{i})^{-1}\leq{}2$, we obtain
\begin{align*}
&\theta{}_i = (B_{i}s_{i})^{-1}\log(i^{2}(\bar{s}/s_{i}))+1 \leq {} 2\log(i^{2}(\bar{s}/s_{i}))+1 = 2\log(i^{2}2^{i-1})+1 \leq 2\log{}(i^{2}2^{i}) \\
& \leq \frac{C}{4}\log\prn*{\frac{\sqrt{C}\log{}N}{\alpha{}_i}}.
\end{align*}
Finally, we have
\[
\sup_{\substack{g\in{}\G\\i\in{}\brk{L}}}\brk*{\sum_{t=1}^{n}\eps_{t}g(\z_{t}(\eps)) 
- \frac{\alpha_{i}}{2}\sum_{t=1}^{n}g^{2}(\z_t(\eps))
-\frac{32\log{}N}{\alpha_i}\log{}\prn*{\frac{\sqrt{32}\log{}N}{\alpha_i}}} \leq {} \En\crl*{X_{i} - B_{i}\theta_{i}} \leq \frac{\sqrt{e}}{2} \leq{} 1.
\]
\end{proof}
\end{small}
% !TEX root = paper-plain.tex
\begin{small}
\begin{proof}[\textbf{Proof of Corollary~\ref{cor:generic_rademacher}}]
We prove the corollary for convex Lipschitz loss where we remove the loss function using the symmetrization lemma shown earlier. However even if we consider non-convex classes, the loss is readily removed in the step in the proof below where we apply Lemma \ref{lem:dudley_probability} where the Lipchitz constant is removed when we move to covering numbers. However this is a well known technique and to make the proof simpler we simply assume convexity of loss as well. Our starting point to proving the bounds is Lemma \ref{lem:sym}, Eq. \eqref{eq:general}. To show achievability it suffices to show that 
{\small \begin{align*}
\En_{\epsilon} \sup_{f \in \F} &\sum_{t=1}^n \epsilon_t f(\x_t(\epsilon)) - K_1 \Rad_n(\F(2R(f))) \log^{3/2} n \left( 1 +  \sqrt{\log\left( \frac{\Rad_n(\F(2R(f)))}{\Rad_n(\F(R(1)))} \right) + \log(\log(2R(f)))}\right) \\
&\le K_2 \Gamma \Rad_n(\F(1)) \log^{3/2} n 
\end{align*}}
where $\Gamma$ is the constant that will be inherited from Lemma \ref{lem:dudley_probability}. Define $R_i = 2^i$ and note that since the Rademacher complexity of the class $\F(R)$ is non-decreasing with $R$,
\begin{small}
\begin{align}
&\sup_{f \in \F}  \sum_{t=1}^n \epsilon_t f(\x_t(\epsilon)) - K_1 \Rad_n(\F(2R(f))) \log^{3/2} n \left( 1 +  \sqrt{\log\left( \frac{\Rad_n(\F(2R(f)))}{\Rad_n(\F(1))} \right) + \log(\log(2R(f)))}\right) \notag\\
& = \sup_{R \ge 1} \sup_{f \in \F(R)}  \sum_{t=1}^n \epsilon_t f(\x_t(\epsilon)) - K_1 \Rad_n(\F(2R)) \log^{3/2} n \left( 1 +  \sqrt{\log\left( \frac{\Rad_n(\F(2R))}{\Rad_n(\F(1))} \right) + \log(\log(2R))}\right) \notag\\
& \le \max_{i \in \mathbb{N}} \sup_{f \in \F(R_i)}  \sum_{t=1}^n \epsilon_t f(\x_t(\epsilon)) - K_1 \Rad_n(\F(R_i)) \log^{3/2} n \left( 1 +  \sqrt{\log\left( \frac{\Rad_n(\F(R_i))}{\Rad_n(\F(1))} \right) + \log(\log(R_i))}\right)\label{eq:disc}.
\end{align}\end{small}Denote a shorthand $C_n = \sqrt{96 \log^{3}(en^2) }$ and $D_n^i = \Rad_n(\F(R_i))$. Now note that by Lemma \ref{lem:dudley_probability} we have that for every $i$ and every $\theta > 
1$, 
\[
P_\epsilon\left( \sup_{f\in\F(R_i)} \left|\sum_{t=1}^n \epsilon_t f(\x_t(\epsilon))\right| > 8\left(1+\theta C_n\right)\cdot D_n^i) \right) \leq 2\Gamma e^{- 3 \theta^2} \ .
\]
Let $X_i =  \sup_{f\in\F(R_i)} \left|\sum_{t=1}^n \epsilon_t f(\x_t(\epsilon))\right|$ and let $B_i = 8\left(1+ C_n \right)\cdot D_n^i$. In this case rewriting the above one sided tail bound appropriately (with $\theta = 1+ \tau/(8C_n D_n^i)$) we see that for any $\tau > 0$,
\[
P(X_i - B_i > \tau) \le \frac{2\Gamma}{e^3} \exp\left(- \frac{\tau^2}{2^{8} \log^{3}(en^2) \Rad^2_n(\F(R_i)) } \right).
\]
This establishes one-sided subgaussian tail behavior. Now applying Proposition \ref{prop:main} and setting $\theta_i$ as suggested by the proposition we conclude that
\begin{small}
\begin{align*}
&\Es{\epsilon}{\max_{i \in \mathbb{N}} \sup_{f \in \F(R_i)}  \sum_{t=1}^n \epsilon_t f(\x_t(\epsilon) - K_1 \Rad_n(\F(R_i)) \log^{3/2} n \left( 1 +  \sqrt{\log\left( \frac{\Rad_n(\F(R_i))}{\Rad_n(\F(1))} \right) + \log(\log(R_i))}\right)
} \\
&\le K_2 \Gamma \Rad_n(\F(1)) \log^{3/2} n .
\end{align*}
\end{small}
This concludes the proof by appealing to Eq. \eqref{eq:disc}.
\end{proof}

\begin{proof}[\textbf{Proof of Achievability for Example \ref{eg:linopt_bound}}]
\begin{lemma}
The following bound is achievable in the setting of Example \ref{eg:linopt_bound}:
\[
\B(f) = D\sqrt{n}\prn*{8\nrm{f}\prn*{
1 + \sqrt{\log(2\nrm{f}) + \log\log(2\nrm{f})
}
}
+12}.
\]
\end{lemma}
This proof specializes the proof of Corollary \ref{cor:generic_rademacher} to the regime where Lemma \ref{lem:pinelis} applies.

Recall our parameterization of $\F$: $\F(R) = \crl*{f\in{}\F:\nrm{f}\leq{}R}$. It was shown in \cite{RSS12} that $\mathcal{C}_n(\F(R))\defeq 2RD\sqrt{n}$ is an upper bound for $\Rad_n(\F(R))$. We consider the rate
\[
\B_n(f) = 2\C_n(\F(2R(f)))\prn*{1 + \sqrt{\log\prn*{\frac{\C_n(\F(2R(f)))}{\C_n(\F(1))}} + \log\log_{2}(2R(f))} }.
\]
We begin by applying Lemma \ref{lem:sym}, equation \eqref{eq:convlip}, yielding
\begin{small}
\[
\A_n \leq \sup_{\y}\En_{\eps}\sup_{f}
2\sum_{t=1}^{n}\eps_{t}\tri*{f, \y_{t}(\eps)} - 2\C_n(\F(2R(f)))\prn*{1 + \sqrt{\log\prn*{\frac{\C_n(\F(2R(f)))}{\C_n(\F(1))}} + \log\log_{2}(2R(f))} }.
\]
\end{small}We now discretize the range of $R$ via $R_i = 2^{i}$. By analogy with the proof of Corollary \ref{cor:generic_rademacher} we get the upper bound,
\begin{align*}
&\sup_{\y}\En_{\eps}\sup_{i\in{}\N}\brk*{
\sup_{f\in{}\F(R_{i})}2\sum_{t=1}^{n}\eps_{t}\tri*{f, \y_{t}(\eps)} - 2\C_n(\F(R_{i}))\prn*{1 + \sqrt{\log\prn*{\frac{\C_n(\F(R_i))}{\C_n(\F(1))}} + \log\log_{2}(R_i)} }
} \\
&= 
\sup_{\y}\En_{\eps}\sup_{i\in{}\N}\brk*{
2R_{i}\nrm*{\sum_{t=1}^{n}\eps_{t}\y_{t}(\eps)}_{\star} - 4D\sqrt{n}R_{i}\sqrt{\log\prn*{R_i} + \log(i)}
}.
\end{align*}
Fix a $\Y$-valued tree $\y$ and define a set of random variables $X_{i} = 2R_i\nrm*{\sum_{t=1}^{n}\eps_{t}\y_{t}(\eps)}_{\star}$. Let $B_i = 2D\sqrt{n}R_i$. Lemma \ref{lem:pinelis} shows that
\[
P\prn*{X_i - B_i \geq {}\tau } \leq {}2\exp\prn*{-\frac{\tau^2}{8D^2R_{i}^{2}n}}.
\]
So we have $\sigma_i = 2DR_{i}\sqrt{n}$, and it will be sufficient to set $\bar{\sigma} = 2D\sqrt{n}$. Since our tail bound is purely sub-gaussian, we apply Proposition \ref{prop:main} with $\theta_i = \frac{\sigma_i}{B_i}\sqrt{2\log(\sigma_i/\bar{\sigma}) + 4\log(i)} + 1$, yielding the following bound:
\begin{align*}
\sup_{\y}\En_{\eps}\sup_{i\in{}\N}\brk*{
2R_{i}\nrm*{\sum_{t=1}^{n}\eps_{t}\y_{t}(\eps)}_{\star} - 4D\sqrt{n}R_{i}\sqrt{\log\prn*{R_i} + \log(i)}
} \leq 12D\sqrt{n}.
\end{align*}
\end{proof}
\begin{proof}[\textbf{Proof of Achievability for Example \ref{eg:pacb}}]
Unfortunately, the general symmetrization proof in Lemma \ref{lem:sym} does not suffice for this problem. In what follows we use a more specialized symmetrization technique to prove the lemma.
\begin{lemma}
For any countable class of experts, when we consider $\F$ to be the class of all distributions over the set of experts, the following adaptive bound is achievable:
$$
\B_{n}(f; y_{1:n}) =  \sqrt{50 \left(\mrm{KL}(f|\pi)+\log(n)\right) \sum_{t=1}^n \ip{f}{y_t}} + 50 \left(\mrm{KL}(f|\pi) + \log(n)\right)   + 1.
$$ 
\end{lemma}

To show that the rate is achievable we need to show that $\A_n \le 0$. Since each $\hat{y}_t$ is a distribution over experts and we are in the linear setting, we do not need to randomize in the definition of the minimax value. Let us use the shorthand 
$$C(f) = \mrm{KL}(f|\pi)+\log(n),$$
and take constants $K_1,K_2$ to be determined later. Define
\begin{align*}
\A_n &=  \dtri*{\inf_{\hat{y}_t\in{}\Delta}\sup_{y_t\in{}\Y}}_{t=1}^{n}\Bigg[
\sum_{t=1}^{n} \ip{\hat{y}_t}{y_t} -\inf_{f\in{}\Delta}\Bigg\{\sum_{t=1}^{n}\ip{f}{ y_t} + \sqrt{K C(f) \sum_{t=1}^n \En_{i\sim{}f}\ip{e_{i}}{y_t}^2} + \sqrt{K'} C(f) \Bigg\} \Bigg].
\end{align*}
Using repeated minimax swap, this expression is equal to
\begin{align*}
&\dtri*{\sup_{p_t\in{}\Delta(\Y)} \inf_{\hat{y}_t\in{}\Delta}}_{t=1}^{n}\Bigg[
\sum_{t=1}^{n} \ip{\hat{y}_t}{y_t} -\inf_{f\in{}\Delta}\Bigg\{\sum_{t=1}^{n}\ip{f}{ y_t} + \sqrt{K C(f) \sum_{t=1}^n \En_{i\sim{}f}\ip{e_{i}}{y_t}^2} +  \sqrt{K'} C(f) \Bigg\} \Bigg]\\
 &\begin{aligned}=
\dtri*{\sup_{p_t\in{}\Delta(\Y)} \En_{y_t \sim p_t} }_{t=1}^{n}\Bigg[&
\sum_{t=1}^{n} \inf_{\hat{y}_t\in{}\Delta} \Es{y_t \sim p_t}{\ip{\hat{y}_t}{y_t}} 
\\&-\inf_{f\in{}\Delta}\Bigg\{\sum_{t=1}^{n}\ip{f}{ y_t} + \sqrt{K C(f) \sum_{t=1}^n \En_{i\sim{}f}\ip{e_{i}}{y_t}^2} + \sqrt{K'} C(f) \Bigg\} \Bigg].
\end{aligned}
\end{align*}
By sub-additivity of square-root we pass to an upper bound,
\begin{align*}
\multiminimax{\sup_{p_t} \En_{y_t \sim p_t}}_{t=1}^n \Bigg[ \sup_{f \in \F}& \sum_{t=1}^n \inf_{\hat{y}_t \in \Delta} \Es{y_t\sim{}p_t}{\ip{\hat{y}_t}{y_t}}- \Es{e_i \sim f}{\ip{e_i}{y_t}}\\  &- \sqrt{C(f) \left(K \sum_{t=1}^n {\En_{i\sim{}f}\brk*{\ip{e_{i}}{y_t}^2}} + K' C(f)\right)  } \Bigg].
\end{align*}
We now split the square root according to the formula $\sqrt{ab} = \inf_{\alpha > 0} \left\{ a/2\alpha + \alpha b/2\right\}$ and note the range of the optimal value:
\begin{align}
	\label{eq:alpha_value_KL}
	\frac{1}{\sqrt{n}}\leq \alpha^* = \sqrt{\frac{C(f)}{  \left(K \sum_{t=1}^n \En_{i\sim{}f}\brk*{\ip{e_{i}}{y_t}^2} + K' C(f)\right)} } \le \frac{1}{\sqrt{K'}}.
\end{align}
Let us discretize the interval by setting $\alpha_i = \frac{1}{\sqrt{K'}} 2^{-(i-1)}$ for $i=1,\ldots,N$ and note that we only need to take $N=O(\log(n))$ elements. Write $I=\{\alpha_1,\ldots,\alpha_N\}$. Observe that 
$$\sqrt{ab} = \inf_{\alpha > 0} \left\{ a/2\alpha + \alpha b/2\right\} \geq \min_{\alpha \in I} \left\{ a/4\alpha + \alpha b/2\right\}.$$
For the rest of the proof, the maximum over $\alpha$ is taken within the set $I$. We have
\begin{align}
\A_n \leq \multiminimax{\sup_{p_t} \En_{y_t \sim p_t}}_{t=1}^n \Bigg[& \sup_{f \in  \Delta , \alpha} \sum_{t=1}^n \inf_{\hat{y}_t \in \Delta(\F)} \Es{y_t}{\ip{\hat{y}_t}{y_t}}- \Es{e_i \sim f}{\ip{e_i}{y_t}}  \notag \\ 
&- \frac{\alpha}{2}  \left(K\sum_{t=1}^n \En_{i\sim{}f}\brk*{\ip{e_{i}}{y_t}^2} + K'C(f)\right) -  \frac{C(f)}{4\alpha}\Bigg].\label{eq:setting_alpha_KL}
\end{align}
Dropping some negative terms, we upper bound the last expression by
\begin{align*}
& \multiminimax{\sup_{p_t} \En_{y_t \sim p_t}}_{t=1}^n \Bigg[ \sup_{f \in \F , \alpha} \sum_{t=1}^n \ip{f}{\E{y'_t} - y_t}  -  \frac{K\alpha}{2}  \sum_{t=1}^n \En_{i\sim{}f}\brk*{\ip{e_{i}}{y_t}^2} -  \frac{C(f)}{4\alpha} \Bigg].
\end{align*}
Adding and subtracting $\frac{\alpha }{4} \sum_{t=1}^n \Es{y'_t}{\En_{i\sim{}f}\brk*{\ip{e_{i}}{y'_t}^2}}$, 
\begin{align*}
 \le \multiminimax{\sup_{p_t} \En_{y_t \sim p_t}}_{t=1}^n \Bigg[ \sup_{f \in \F , \alpha} \sum_{t=1}^n \ip{f}{\E{y'_t} - y_t}  - \frac{K\alpha }{4} \sum_{t=1}^n \En_{i\sim{}f}\brk*{\ip{e_{i}}{y_t}^2} - \frac{K\alpha}{4} \sum_{t=1}^n \Es{y'_t}{\En_{i\sim{}f}\brk*{\ip{e_{i}}{y'_t}^2}}& \\
 + \frac{K\alpha}{4} \left(\sum_{t=1}^n \Es{y'_t}{\En_{i\sim{}f}\brk*{\ip{e_{i}}{y'_t}^2}}  - \En_{i\sim{}f}\brk*{\ip{e_{i}}{y_t}^2}  \right) -  \frac{C(f)}{4\alpha}& \Bigg].
\end{align*}
%\begin{align*}
%& \le \multiminimax{\sup_{p_t} \En_{y_t \sim p_t}}_{t=1}^n \Bigg[ \sup_{f \in \F , \alpha} \sum_{t=1}^n \ip{f}{\E{y'_t} - y_t}  - \frac{K\alpha }{4} \sum_{t=1}^n \En_{i\sim{}f}\brk*{\ip{e_{i}}{y_t}^2} - \frac{K\alpha}{4} \sum_{t=1}^n \Es{y'_t}{\En_{i\sim{}f}\brk*{\ip{e_{i}}{y'_t}^2}} \\
%&~~~~~~~~~~~~~~~~~~~~~~~~~ ~~~~~ + \frac{K\alpha}{4} \left(\sum_{t=1}^n \Es{y'_t}{\En_{i\sim{}f}\brk*{\ip{e_{i}}{y'_t}^2}}  - \En_{i\sim{}f}\brk*{\ip{e_{i}}{y_t}^2}  \right) -  \frac{C(f)}{4\alpha} \Bigg].
%\end{align*}
Using Jensen's inequality to pull out expectations, we obtain an upper bound,
\begin{align*}
\multiminimax{\sup_{p_t} \En_{y_t,y'_t \sim p_t}}_{t=1}^n \Bigg[ \sup_{f \in \F , \alpha} \sum_{t=1}^n \ip{f}{y'_t - y_t}  - \frac{K\alpha }{4} \sum_{t=1}^n \En_{i\sim{}f}\brk*{\ip{e_{i}}{y_t}^2} - \frac{K\alpha}{4} \sum_{t=1}^n \En_{i\sim{}f}\brk*{\ip{e_{i}}{y'_t}^2}& \\
+ \frac{K\alpha}{4} \left(\sum_{t=1}^n \En_{i\sim{}f}\brk*{\ip{e_{i}}{y'_t}^2}  - \En_{i\sim{}f}\brk*{\ip{e_{i}}{y_t}^2}  \right) -  \frac{C(f)}{4\alpha}& \Bigg].
\end{align*}
Next, we introduce Rademacher random variables:
\begin{align*}
& \multiminimax{\sup_{p_t} \En_{y_t,y'_t \sim p_t}\En_{\eps_t}}_{t=1}^n \Bigg[ \sup_{f \in \F , \alpha} \sum_{t=1}^n \eps_{t}\prn*{\ip{f}{y'_t - y_t} + \frac{K\alpha}{4}\prn*{\En_{i\sim{}f}\brk*{\ip{e_{i}}{y'_t}^2}  - \En_{i\sim{}f}\brk*{\ip{e_{i}}{y_t}^2}  } } \\
&~~~~~~~~~~~~~~~~~~~~~~~~~ ~~~~~  - \frac{K\alpha }{4} \sum_{t=1}^n \En_{i\sim{}f}\brk*{\ip{e_{i}}{y_t}^2} - \frac{K\alpha}{4} \sum_{t=1}^n \En_{i\sim{}f}\brk*{\ip{e_{i}}{y'_t}^2} -  \frac{C(f)}{4\alpha} \Bigg]\\
&\leq \multiminimax{\sup_{y_t} \En_{\eps_t}}_{t=1}^n \Bigg[ \sup_{f \in \F , \alpha} \sum_{t=1}^n \eps_{t}\prn*{2\ip{f}{y_t} + \frac{K\alpha}{2}\En_{i\sim{}f}\brk*{\ip{e_{i}}{y_t}^2} } - \frac{K\alpha }{2} \sum_{t=1}^n \En_{i\sim{}f}\brk*{\ip{e_{i}}{y_t}^2} -  \frac{C(f)}{4\alpha} \Bigg].
\end{align*}
Moving to the tree notation, we have
\begin{align*}
	\sup_{\y}\En_{\epsilon}\sup_{f \in \F , \alpha}\biggl[& \sum_{t=1}^n \eps_{t}\prn*{2\ip{f}{\y_t(\eps)} + \frac{K\alpha}{2}\En_{i\sim{}f}\brk*{\ip{e_{i}}{\y_t(\eps)}^2} } \\
	&- \frac{K\alpha }{2} \sum_{t=1}^n \En_{i\sim{}f}\brk*{\ip{e_{i}}{\y_t(\eps)}^2}  -  \frac{\mrm{KL}(f | \pi)}{4\alpha} - \frac{ \log{}n}{4\alpha}\biggr].
\end{align*}
Noting that the convex conjugate of $\frac{1}{\alpha}\mrm{KL}(f \| \pi)$ is given by $\Psi^*(X) = \frac{1}{\alpha}\log\left( \En_{i \sim \pi}{\exp\left(\alpha \ip{e_i}{X} \right)} \right)$, we express the last quantity as
\begin{align*}
& \sup_{\y}\En_{\epsilon}{\max_{\alpha} \frac{1}{4\alpha} \log\left(\En_{i \sim \pi} \exp\left(  \sum_{t=1}^n \eps_{t}\prn*{8\alpha\tri{e_i, \y_t(\eps)} + 2K\alpha^{2}\tri{e_i, \y_t(\eps)}^2 }  -  2K\alpha^2 \ip{e_i}{\y_t(\epsilon)}^2 \right)\right) - \frac{ \log{}n}{4 \alpha} }.
\end{align*}
Define a random variable indexed by $\alpha$:
\[
\textstyle X_\alpha  = \frac{1}{4\alpha} \log\left(\Es{i \sim \pi}{ \exp\left(  \sum_{t=1}^n \eps_{t}\prn*{8\alpha\tri{e_i, \y_t(\eps)} + 2K\alpha^{2}\tri{e_i, \y_t(\eps)}^2 }  -  2K\alpha^2 \ip{e_i}{\y_t(\epsilon)}^2 \right)}\right).
\] 
Our goal is to bound $\E{\max_{\alpha}\{ X_\alpha - \log(n)/4\alpha\}}$. Now notice that
\begin{align*}
 P(X_\alpha > t) &\le \inf_{\lambda} \E{e^{\lambda X_\alpha - \lambda t}}\\
 & = \inf_{\lambda}\left\{ \En_{\epsilon}{\left(\En_{i \sim \pi}{ \exp\left(  \sum_{t=1}^n \eps_{t}\prn*{8\alpha\tri{e_i, \y_t(\eps)} + 2K\alpha^{2}\tri{e_i, \y_t(\eps)}^2 }  -  2K\alpha^2 \ip{e_i}{\y_t(\epsilon)}^2 \right)}\right)^{\frac{\lambda}{4\alpha}} }e^{-\lambda t}\right\}\\
  & \le  \En_{\epsilon}{\En_{i \sim \pi}{ \exp\left(  \sum_{t=1}^n \eps_{t}\prn*{8\alpha\tri{e_i, \y_t(\eps)} + 2K\alpha^{2}\tri{e_i, \y_t(\eps)}^2 }  - 2K\alpha^2 \ip{e_i}{\y_t(\epsilon)}^2 \right)} }e^{-4\alpha t}\\
    & \le  \En_{\epsilon}{\En_{i \sim \pi}{ \exp\left(  \sum_{t=1}^n \prn*{8\alpha\tri{e_i, \y_t(\eps)} + 2K\alpha^{2}\tri{e_i, \y_t(\eps)}^2 }^2  - 2K\alpha^2 \ip{e_i}{\y_t(\epsilon)}^2 \right)} }e^{-4\alpha t}\\
     & \le  \En_{\epsilon}{\En_{i \sim \pi}{ \exp\left(  \sum_{t=1}^n 4\dfedit{\alpha^2{}}\prn*{4 + K\alpha}^2\tri{e_i, \y_t(\eps)}^2  - 2K\alpha^2 \ip{e_i}{\y_t(\epsilon)}^2 \right)} }e^{-4\alpha t}.
\end{align*}
The above term is upper bounded by $\exp(-4\alpha t)$ as soon as $4\alpha^2(4+K\alpha)^2\leq 2K\alpha^2$, which happens when 
\begin{align}
	\label{eq:desired_alpha_1}
	0 < \alpha\leq (\sqrt{K/2}-4)/K.
\end{align}
In view of \eqref{eq:alpha_value_KL}, we know that
$
\alpha \le \frac{1}{\sqrt{K'}}. 
$
Thus, to ensure \eqref{eq:desired_alpha_1}, it is sufficient to take $K=50$ and $K'=50^2$. Other choices lead to a different balance of constants. We thus have
\begin{align*}
 P(X_\alpha > t) &\le \exp\left( - 4\alpha t\right).
\end{align*}
Now that we have the tail bound, we appeal to Proposition \ref{prop:main}. Setting $s_i=4\alpha_i$ and $B_i=1/4\alpha_i$, we obtain that
$$
 \E{\max_{i=1,\ldots, N}\left\{ X_{\alpha_i} - \frac{\log(n)}{4\alpha} \right\}}  \le 10.
$$
\end{proof}
\section{Relaxations and Algorithms}
\label{app:algo}
\begin{proof}[\textbf{Proof of Admissibility for Example \ref{eg:kl_alg}}]
\begin{lemma}
The following bound is achievable in the setting given in example \ref{eg:kl_alg}:
\begin{equation}
\label{eq:kl_alg_bound}
\B_{n}(f) = 3\sqrt{2n\max\crl*{\KL(f\mid{}\pi),1}} + 4\sqrt{n}.
\end{equation}
\end{lemma}

 This algorithm can be interpreted as running a ``low-level'' instance of the exponential weights algorithm for each complexity radius $R_{i}$, then combining the predictions of these algorithms with a ``high-level'' instance. The high-level distribution $q_t^{\star}$ differs slightly from the usual exponential weights distribution in that it incorporates a prior whose weight decreases as the complexity radius increases. The prior distribution prevents the strategy from incurring a penalty that depends on the range of values the complexity radii take on, which would happen if the standard exponential weights distribution were used. 
 
Following the analysis style of Corollary \ref{cor:generic_rademacher}, we directly consider an upper bound based on $\KL(f\mid{}\pi)$ but instead use a complexity-radius-based upper bound with the KL divergence controlling the complexity radius: $\F(R) = \crl*{f:\KL(f\mid{}\pi)\leq{}R}$. Concretely, we move from \eqref{eq:kl_alg_bound}
to the bound 
\[\B_{n}(i) = 3\sqrt{nR_{i}} + 4\sqrt{n}\] 
for $R_{i}=2^{i-1}$ with $i\in{}\N$. To keep the analysis as tidy as possible, we will study the achievability of $\B_n(i) = D\sqrt{R_{i}n}$, setting $D$ and including additive constants only when we reach a point in the analysis where it becomes necessary to do so. The relaxation we consider is
\begin{align*}
\Bay(\yr[t]) = \inf_{\lambda>0}\brk*{
\frac{1}{\lambda}\log\prn*{\sum_{i}\exp\prn*{
    -\lambda\brk*{\sum_{s=1}^{t}\tri*{q_{s}^{R_i}(\yr[s-1]),y_s} -  2\sqrt{nR_{i}} + \B_{n}(i)
}
}
}
+ 2\lambda(n-t)
}.
\end{align*}
\paragraph{Initial Condition:} This inequality follows from Lemma \ref{lemma:fixed_r} and an application of the softmax function as an upper bound on the supremum over $i$:
  \begin{align*}
&-\inf_{i}\brk*{
    \inf_{f\in{}\F(R_{i})}\sum_{t=1}^{n}\ls(f,y_t) + \B_n(i)
}\\
&\dfedit{\leq} \sup_{i}\brk*{-\sum_{s=1}^{t}\tri*{q_{s}^{R_{i}}(\yr[s-1]),y_s} +
      2\sqrt{nR_{i}}
  - \B_n(i)
}
\\
&\leq \inf_{\lambda{}>0}\frac{1}{\lambda}\log\prn*{\sum_{i}\exp\prn*{-\lambda{}\brk*{\sum_{s=1}^{t}\tri*{q_{s}^{R_i}(\yr[s-1]),y_s} -
      2\sqrt{nR_{i}}
  + \B_n(i)
}}}
\\
&= \Bay(\yr[n]).
  \end{align*}
\paragraph{Admissibility Condition:}
Define a strategy $q_{t}^{\star}$ via
  \begin{equation*}
    (q_{t}^{\star})_{i} = \frac{\exp\prn*{-\lambda{}_{t}^{\star}\brk*{\sum_{s=1}^{t-1}\tri*{q_{s}^{R_i}(\yr[s-1]),y_s} -
      2\sqrt{nR_{i}}
	 + \B_n(i)
}}}{\sum_{j}\exp\prn*{-\lambda{}_{t}^{\star}\brk*{\sum_{s=1}^{t-1}\tri*{q_{s}^{R_j}(\yr[s-1]),y_s} -
      2\sqrt{nR_{j}}
  + \B_n(R_{j})
}}},
\end{equation*}
where we have set
\begin{align*}
\lambda_{t}^{\star} = \argmin_{\lambda>0}\brk*{\frac{1}{\lambda}\log\prn*{\sum_{i}\exp\prn*{-\lambda{}\brk*{\sum_{s=1}^{t-1}\tri*{q_{s}^{R_i}(\yr[s-1]),y_s} -
      2\sqrt{nR_{i}}
	 + \B_n(i)
}}}
+ 2\lambda(n-t + 1)}.
\end{align*}
We proceed to demonstrate admissibility:
\begin{align*}
&\inf_{q_t}\sup_{y_t}\brk*{
\tri*{q_t, y_t} + 
    \Bay(\yr[t])
}    \\
&= 
\inf_{q_t}\sup_{y_t}\brk*{
\tri*{q_t, y_t} + 
 \inf_{\lambda{}>0}\brk*{\frac{1}{\lambda}\log\prn*{\sum_{i}\exp\prn*{-\lambda{}\brk*{\sum_{s=1}^{t}\tri*{q_{s}^{R_i}(\yr[s-1]),y_s} -
      2\sqrt{nR_{i}}
  + \B_n(i)
}}}
+ 2\lambda(n-t)
}
}.
\intertext{We now plug in $q_{t}^{\star}$ and $\lambda_{t}^{\star}$ as described above:}
&\leq
\sup_{y_t}\biggr[
\frac{1}{\lambda_{t}^{\star}}\log\prn*{
\exp\prn*{
\lambda_{t}^{\star}\En_{i\sim{}q_t^{\star}}\tri*{q_{t}^{R_{i}}(\yr[t-1]), y_t
}
}
}+ \frac{1}{\lambda_{t}^{\star}}\log\prn*{
\En_{i\sim{}q_{t}^{\star}}\exp\prn*{
-\lambda_{t}^{\star}\tri*{q_{t}^{R_{i}}(\yr[t-1]), y_t
}
}
}\\&\hspace{.4in} 
+ \frac{1}{\lambda_{t}^{\star}}\log\prn*{\sum_{i}\exp\prn*{-\lambda{}_t^{\star}\brk*{\sum_{s=1}^{t-1}\tri*{q_{s}^{R_i}(\yr[s-1]),y_s} -
      2\sqrt{nR_{i}}
	 + \B_n(i)
}}}
+ 2\lambda_{t}^{\star}(n-t)
\biggl].
\intertext{We combine the first two terms in the expression and apply Jensen's inequality to arrive at an upper bound:}
&\leq
\sup_{y_t}\biggr[
\frac{1}{\lambda_{t}^{\star}}\log\prn*{
\En_{i,i'\sim{}q_t^{\star}}
\exp\prn*{
\lambda_{t}^{\star}\tri*{q_{t}^{R_{i}}(\yr[t-1]) - q_{t}^{R_{i'}}(\yr[t-1]), y_t
}
}
}\\&\hspace{0.4in}
+ \frac{1}{\lambda_{t}^{\star}}\log\prn*{\sum_{i}\exp\prn*{-\lambda{}_t^{\star}\brk*{\sum_{s=1}^{t-1}\tri*{q_{s}^{R_i}(\yr[s-1]),y_s} -
      2\sqrt{nR_{i}}
	 + \B_n(i)
}}}
+ 2\lambda_{t}^{\star}(n-t)
\biggl].
\intertext{The first term is now bounded using sub-gaussianity.}
&\leq
  \frac{1}{\lambda_{t}^{\star}}\log\prn*{\sum_{i}\exp\prn*{-\lambda{}_t^{\star}\brk*{\sum_{s=1}^{t-1}\tri*{q_{s}^{R_i}(\yr[s-1]),y_s} -
      2\sqrt{nR_{i}}
	 + \B_n(i)
}}}
+ 2\lambda_{t}^{\star}(n-t + 1)
\\
&=
\inf_{\lambda>0}\brk*{\frac{1}{\lambda}\log\prn*{\sum_{i}\exp\prn*{-\lambda{}\brk*{\sum_{s=1}^{t-1}\tri*{q_{s}^{R_i}(\yr[s-1]),y_s} -
      2\sqrt{nR_{i}}
	 + \B_n(i)
}}}
+ 2\lambda(n-t + 1)}
\\
&= \Bay(\yr[t-1]).
\end{align*}

Having shown that $\Bay$ is an admissible relaxation, it remains to show that the relaxation's final value,
\begin{align*}
  \Bay(\cdot) &
= \inf_{\lambda>0} 
\brk*{\frac{1}{\lambda}\log\prn*{
\sum_{i}\exp\prn*{\lambda\brk*{
      2\sqrt{nR_{i}}
  - D\sqrt{nR_{i}}
}}
}
+ 2\lambda{}n
}
\intertext{is not too large. Setting $D$ = 3,}
&= \inf_{\lambda>0} 
\brk*{\frac{1}{\lambda}\log\prn*{
\sum_{i}\exp\prn*{-\lambda\sqrt{nR_{i}}
}}
+ 2\lambda{}n
}.
\intertext{The complexity radius $R_{i}$ is discretized such that $R_{i} -
  R_{i-1}\geq{}1$, yielding}
&\leq \inf_{\lambda>0} 
\brk*{\frac{1}{\lambda}\log\prn*{
\exp(-\lambda\sqrt{n})+\sum_{i=2}^{\infty}(R_{i}-R_{i-1})\exp\prn*{-\lambda\sqrt{nR_{i}}
}}
+ 2\lambda{}n
}
\\
&\leq \inf_{\lambda>0} 
\brk*{\frac{1}{\lambda}\log\prn*{\exp(-\lambda\sqrt{n})+
\int_{1}^{\infty}\exp\prn*{-\lambda\sqrt{nR}
}dR}
+ 2\lambda{}n
}.
\end{align*}
The integral is a routine calculation.
\begin{align*}
  \int_{1}^{\infty}\exp\prn*{-\lambda\sqrt{nR}
}dR & =
      -2\frac{1}{\lambda^{2}n}\exp\prn*{-\lambda\sqrt{nR}}\brk*{\lambda{}\sqrt{nR} +1
}\bigg|_{1}^{\infty}.
\end{align*}
Finally, set $\lambda{} = 1/\sqrt{n}$ yielding
\begin{align*}
\Bay(\cdot) 
&\leq 4\sqrt{n}.
\end{align*}
Note that instead of setting $\lambda_{t}= \lambda_{t}^{\star}$ as described above, we could have set $\lambda_{t} = 1/\sqrt{n}$ and achieved the same regret bound.
\end{proof}
\begin{lemma}
\label{lemma:fixed_r}
Consider the experts setting from Example \ref{eg:pacb}, but with hypothesis class $\F(R) = \crl*{f:\KL(f\mid{}\pi)\leq{}R}$. The following inequality holds:
\[
  -\inf_{f\in{}\F(R)}\sum_{t=1}^{n}\tri{y_t,f} \leq -\sum_{t=1}^{n}\tri*{y_t, q^{R}(\yr[t-1])} + 2\sqrt{Rn}.
\]
\end{lemma}

\begin{proof}
Our strategy is to move to an upper bound based on the Kullback-Leibler divergence and exploit convex duality:
\begin{align*}
  &-\inf_{f\in{}\F(R)}\sum_{t=1}^{n}\tri{y_t,f}\\
  &\leq -\inf_{f\in{}\F(R)}\crl*{\sum_{t=1}^{n}\tri{y_t,f} + \alpha\KL(f\mid{}\pi)} + \alpha R \\
  &\leq -\inf_{f\in{}\F}\crl*{\sum_{t=1}^{n}\tri{y_t,f} + \alpha\KL(f\mid{}\pi) } + \alpha R.
  \intertext{We use $\Psi^{\star}$ to denote the Fenchel conjugate of $\KL(\cdot\mid{}\pi)$:}
  &= \alpha\Psi^{\star}\prn*{-\frac{1}{\alpha}\sum_{t=1}^{n}y_t} \Psi+ \alpha R.
  \end{align*}
  The function $\KL(\cdot\mid{}\pi)$ is $1$-strongly convex, which implies that $\Psi^*$ is $1$-strongly smooth. We peel off one term at a time:
  \begin{align*}
	  \alpha\Psi^{\star}\prn*{-\frac{1}{\alpha}\sum_{t=1}^{n}y_t} \leq \alpha\Psi^{\star}\prn*{-\frac{1}{\alpha}\sum_{t=1}^{n-1}y_t} + \inner{-y_n, \nabla\Psi^{\star}\prn*{-\frac{1}{\alpha}\sum_{t=1}^{n-1}y_t}} + \frac{1}{\alpha}.
  \end{align*}
This obtains the following upper bound:
  \begin{align*}
  &
    -\sum_{t=1}^{n}\tri*{y_t, \nabla\Psi^{\star}\prn*{-\frac{1}{\alpha}\sum_{s=1}^{t-1}y_s}} + \frac{KCn}{\alpha}
+ \alpha R.
  \end{align*}
  Setting $\alpha = \sqrt{n/R}$ and noting that $\nabla\Psi^{\star}\prn*{-\sqrt{\frac{R}{n}}\sum_{s=1}^{t-1}y_s} = q^{R}(\yr[t-1])$ yields the result.
\end{proof}

\begin{proof}[\textbf{Proof of Lemma \ref{lemma:ada}}]

Recall the form of the $\Ada$ relaxation, where we have abbreviated $\Relaxn^{R}$ to $\Rd^{R}$:
\[
\Ada(\yr[t]) = \sup_{\y,\y'}\En_{\eps}\sup_{R}\brk*{
\Rd^{R}(\yr[t]) - \Rd^{R}\theta(\Rd^{R})
+ 2\sum_{s=t+1}^{n}\eps_{s}\En_{\hat{y}_{s}\sim{}q_{s}^{R}(\yr[t],\y'_{t+1:s-1}(\eps))}\ls(\hat{y}_s,\y_s(\eps))
}.
\]

\paragraph{Initial Condition:}

This directly follows from the fact that $\Rd^R$ satisfy the initial condition:

\begin{align*}
  \Ada(\yr[n]) &= \sup_{R}\brk*{
\Rd^{R}(\yr[n]) - \Rd^{R}\theta(\Rd^{R})
}\\
&\geq
\sup_{R}\brk*{
-\inf_{f\in{}\F(R)}\sum_{t=1}^{n}\ls(f, y_t) - \Rd^{R}\theta(\Rd^{R})
}\\
&=
-\inf_{R}\inf_{f\in{}\F(R)}\brk*{
\sum_{t=1}^{n}\ls(f, y_t) + \Rd^{R}\theta(\Rd^{R})
}.
\end{align*}
Therefore, playing the strategy corresponding to $\Ada$ yields an adaptive regret bound of the form $\B_{n}(R) = \Relaxn^{R}(\cdot)\theta(\Relaxn^{R}(\cdot)) + \Ada(\cdot)$.\\
~\\
\noindent\textbf{Admissibility Condition:} 
We obtain the following equalities using the same minimax swap technique as in the Lemma \ref{lem:sym} proof:
\begin{align*}
  &\inf_{q_t}\sup_{y_t}\En_{\hat{y}_t\sim{}q_t}\brk*{
  \ls(\hat{y}_t, y_t) + \Ada(\yr[t])
}
\\
&\begin{aligned}= 
\inf_{q_t}\sup_{y_t}\En_{\hat{y}_t\sim{}q_t}\sup_{\y,\y'}\En_{\eps}\sup_{R}\biggl[&
  \ls(\hat{y}_t, y_t) +
\Rd^{R}(\yr[t]) - \Rd^{R}\theta(\Rd^{R})\\
&+ 2\sum_{s=t+1}^{n}\eps_{s}\En_{\hat{y}_{s}\sim{}q_{s}^{R}(\yr[t],\y'_{t+1:s-1}(\eps))}\ls(\hat{y}_s,\y_s(\eps))
\biggr]\end{aligned}
\\
&\begin{aligned}= 
\sup_{p_t}\En_{y_t\sim{}p_t}\sup_{\y,\y'}\En_{\eps}\sup_{R}\biggl[&
  \inf_{\hat{y}_t}\En_{y'_t\sim{}p_t}\ls(\hat{y}_t, y'_t) +
\Rd^{R}(\yr[t]) - \Rd^{R}\theta(\Rd^{R})\\
&+ 2\sum_{s=t+1}^{n}\eps_{s}\En_{\hat{y}_{s}\sim{}q_{s}^{R}(\yr[t],\y'_{t+1:s-1}(\eps))}\ls(\hat{y}_s,\y_s(\eps))
\biggr].\end{aligned}
\end{align*}
 Note that $\inf_{\hat{y}_t}\En_{y'_t\sim{}p_t}\ls(\hat{y}_t, y'_t) = \inf_{q_t\in\Delta(\D)}\En_{\hat{y}_t\sim{}q_t}\En_{y'_t\sim{}p_t}\ls(\hat{y}_t, y'_t),$
  and so we may upper bound using the randomized strategy $q_{t}^{R}$ corresponding to $\Relaxn^{R}$. That this strategy depends on $\yr[t-1]$ is left implicit. This yields an upper bound,
\begin{align*}
\sup_{p_t}\En_{y_t\sim{}p_t}\sup_{\y,\y'}\En_{\eps}\sup_{R}\biggl[
  \En_{y'_t\sim{}p_t}\En_{\hat{y}_t\sim{}q_{t}^{R}}\ls(\hat{y}_t, y'_t) +
\Rd^{R}(\yr[t]) - \Rd^{R}\theta(\Rd^{R})& \\
+ 2\sum_{s=t+1}^{n}\eps_{s}\En_{\hat{y}_{s}\sim{}q_{s}^{R}(\yr[t],\y'_{t+1:s-1}(\eps))}\ls(\hat{y}_s,\y_s(\eps))
&\biggr],
\end{align*}
which we can write by adding and subtracting $\En_{\hat{y}_t\sim{}q_{t}^{R}}\ls(\hat{y}_t, y_t)$ as
\begin{align*}
\sup_{p_t}\En_{y_t\sim{}p_t}\sup_{\y,\y'}\En_{\eps}\sup_{R} \biggl[&
  \En_{y'_t\sim{}p_t}\En_{\hat{y}_t\sim{}q_{t}^{R}}\ls(\hat{y}_t, y'_t) -
  \En_{\hat{y}_t\sim{}q_{t}^{R}}\ls(\hat{y}_t, y_t) +  \En_{\hat{y}_t\sim{}q_{t}^{R}}\ls(\hat{y}_t, y_t)  \\
	&+\Rd^{R}(\yr[t]) - \Rd^{R}\theta(\Rd^{R})+ 2\sum_{s=t+1}^{n}\eps_{s}\En_{\hat{y}_{s}\sim{}q_{s}^{R}(\yr[t],\y'_{t+1:s-1}(\eps))}\ls(\hat{y}_s,\y_s(\eps))\biggr].
\end{align*}
Now, using the fact that $\Rd^{R}$ are admissible,
\begin{align*}
\leq
\sup_{p_t}\En_{y_t\sim{}p_t}\sup_{\y,\y'}\En_{\eps}\sup_{R}\biggl[&
  \En_{y'_t\sim{}p_t}\En_{\hat{y}_t\sim{}q_{t}^{R}}\ls(\hat{y}_t, y'_t) -
  \En_{\hat{y}_t\sim{}q_{t}^{R}}\ls(\hat{y}_t, y_t)  \\
&+\Rd^{R}(\yr[t-1]) - \Rd^{R}\theta(\Rd^{R})+ 2\sum_{s=t+1}^{n}\eps_{s}\En_{\hat{y}_{s}\sim{}q_{s}^{R}(\yr[t],\y'_{t+1:s-1}(\eps))}\ls(\hat{y}_s,\y_s(\eps))\biggr].
\end{align*}
By Jensen's inequality, we upper bound the last expression by
\begin{align*}
\sup_{p_t}\En_{y_t,y'_t\sim{}p_t}\sup_{\y,\y'}\En_{\eps}\sup_{R}\biggl[
\En_{\hat{y}_t\sim{}q_{t}^{R}}\ls(\hat{y}_t, y'_t) -
  \En_{\hat{y}_t\sim{}q_{t}^{R}}\ls(\hat{y}_t, y_t) +
\Rd^{R}(\yr[t-1]) - \Rd^{R}\theta(\Rd^{R})& \\
+ 2\sum_{s=t+1}^{n}\eps_{s}\En_{\hat{y}_{s}\sim{}q_{s}^{R}(\yr[t],\y'_{t+1:s-1}(\eps))}\ls(\hat{y}_s,\y_s(\eps))&
\biggr].
\end{align*}
We now replace each choice $y_t$ in the last sum by a worst-case choice $y''_t$:
\begin{align*}
\leq \sup_{p_t}\En_{y_t,y'_t\sim{}p_t}\sup_{y''_t}\sup_{\y,\y'}\En_{\eps}\sup_{R}\biggl[
\En_{\hat{y}_t\sim{}q_{t}^{R}}\ls(\hat{y}_t, y'_t) -
  \En_{\hat{y}_t\sim{}q_{t}^{R}}\ls(\hat{y}_t, y_t) +
\Rd^{R}(\yr[t-1]) - \Rd^{R}\theta(\Rd^{R})& \\
+ 2\sum_{s=t+1}^{n}\eps_{s}\En_{\hat{y}_{s}\sim{}q_{s}^{R}(\yr[t-1], y''_t,\y'_{t+1:s-1}(\eps))}\ls(\hat{y}_s,\y_s(\eps))&
\biggr].
\end{align*}
 We then introduce $\epsilon_t$ since $y_t,y'_t$ can be renamed. The last expression is equal to
\begin{align*}
\sup_{p_t}\En_{y_t,y'_t\sim{}p_t}\En_{\eps_{t}}\sup_{y''_t}\sup_{\y,\y'}\En_{\eps}\sup_{R}\biggl[
\En_{\hat{y}_t\sim{}q_{t}^{R}}\brk*{\eps_{t}(\ls(\hat{y}_t, y'_t) - \ls(\hat{y}_t, y_t))} +
\Rd^{R}(\yr[t-1]) - \Rd^{R}\theta(\Rd^{R})& \\
+ 2\sum_{s=t+1}^{n}\eps_{s}\En_{\hat{y}_{s}\sim{}q_{s}^{R}(\yr[t-1],y''_t,\y'_{t+1:s-1}(\eps))}\ls(\hat{y}_s,\y_s(\eps))&
\biggr].
\end{align*}
By splitting into two terms we arrive at an upper bound of
\begin{align*}
&\begin{aligned}\sup_{p_t}\En_{y_t\sim{}p_t}\En_{\eps_{t}}\sup_{y''_t}\sup_{\y,\y'}\En_{\eps}\sup_{R}\biggl[&
  2\eps_{t}\En_{\hat{y}_t\sim{}q_{t}^{R}}\brk*{\ls(\hat{y}_t, y_t)} +
\Rd^{R}(\yr[t-1]) - \Rd^{R}\theta(\Rd^{R}) \\
&+ 2\sum_{s=t+1}^{n}\eps_{s}\En_{\hat{y}_{s}\sim{}q_{s}^{R}(\yr[t-1],y''_t,\y'_{t+1:s-1}(\eps))}\ls(\hat{y}_s,\y_s(\eps))
\biggr]\end{aligned}
\\
&\begin{aligned}=
\sup_{y_t}\En_{\eps_{t}}\sup_{y''_t}\sup_{\y,\y'}\En_{\eps}\sup_{R}\biggl[&
  2\eps_{t}\En_{\hat{y}_t\sim{}q_{t}^{R}}\brk*{\ls(\hat{y}_t, y_t)} +
\Rd^{R}(\yr[t-1]) - \Rd^{R}\theta(\Rd^{R}) \\
&+ 2\sum_{s=t+1}^{n}\eps_{s}\En_{\hat{y}_{s}\sim{}q_{s}^{R}(\yr[t-1],y''_t,\y'_{t+1:s-1}(\eps))}\ls(\hat{y}_s,\y_s(\eps))
\biggr]\end{aligned}
\\
&= \Ada(\yr[t-1]).
\end{align*}

\end{proof}
\end{small}

\end{document}